\pdfoutput=1

\documentclass[11pt]{article}

\usepackage{amsmath}
\usepackage{amssymb}
\usepackage{mathtools}
\usepackage{amsthm}

\newtheorem{theorem}{Theorem}
\newtheorem{corollary}{Corollary}
\newtheorem{proposition}{Proposition}
\newtheorem{definition}{Definition}

\newtheorem{claim}{Claim}
\newtheorem{lemma}{Lemma}
\newtheorem*{remark}{Remark}

\newcommand{\cross}{\mathrm{CE}}
\newcommand{\kl}{\mathrm{KL}}
\newcommand{\R}{\mathbb{R}}

\newcommand{\cX}{\mathcal{X}}
\newcommand{\cY}{\mathcal{Y}}

\newcommand{\cC}{\mathcal{C}}

\newcommand{\expect}{\mathbb{E}}

\newcommand{\cO}{\mathcal{O}}

\newcommand{\cF}{\mathcal{F}}
\newcommand{\cP}{\mathcal{P}}
\newcommand{\bE}{\mathbb{E}}

\newcommand{\argmin}{\text{argmin}}
\newcommand{\eps}{\varepsilon}
\newcommand{\tv}{\mathrm{D}_{\mathrm{TV}}}

\newcommand{\dist}{L}
\newcommand{\disthat}{\hat{L}}

\usepackage{cancel}

\makeatletter
\renewcommand{\tableofcontents}{%
    \section*{\contentsname}%
    \begingroup
    \setlength{\baselineskip}{20pt} 
    \@starttoc{toc}%
    \endgroup
}
\makeatother

\allowdisplaybreaks[4]

\usepackage{subfigure}
\usepackage{booktabs}
\usepackage{subcaption}
\usepackage{float}

\usepackage{multirow} 

\usepackage{adjustbox}

\usepackage[final]{acl}

\usepackage{url}

\usepackage{times}
\usepackage{latexsym}

\usepackage[T1]{fontenc}

\usepackage[utf8]{inputenc}

\usepackage{microtype}

\usepackage{inconsolata}

\usepackage{graphicx}

\usepackage[capitalize,noabbrev]{cleveref}

\title{Revisiting Weak-to-Strong Generalization in Theory and Practice: \\ Reverse KL vs. Forward KL}

\author{
\textbf{Wei Yao\textsuperscript{1{$\star$}}},
\textbf{Wenkai Yang\textsuperscript{1{$\star$}}},
\textbf{Ziqiao Wang\textsuperscript{2}},
\textbf{Yankai Lin\textsuperscript{1}},
\textbf{Yong Liu\textsuperscript{1}$^{\dag}$}
\\
\\
 \textsuperscript{1}Gaoling School of Artificial Intelligence, Renmin University of China,
\\
 \textsuperscript{2}School of Computer Science and Technology, Tongji University,
\\
\tt\footnotesize\{wei.yao, wenkaiyang, yankailin, liuyonggsai\}@ruc.edu.cn~~ziqiaowang@tongji.edu.cn\\
}

\begin{document}
\maketitle

\let\thefootnote\relax\footnotetext{$^\star$ Equal contribution\hspace{3pt} \hspace{5pt}$^{\dag}$ Corresponding author\hspace{5pt}}

\begin{abstract}
As large language models advance toward superhuman performance, ensuring their alignment with human values and abilities grows increasingly complex. Weak-to-strong generalization offers a promising approach by leveraging predictions from weaker models to guide stronger systems, but its effectiveness could be constrained by the inherent noise and inaccuracies in these weak predictions. To address this, we propose a theoretically grounded approach that replaces forward KL divergence—whose mass-covering behavior risks overfitting to imperfect weak signals—with reverse KL divergence. Reverse KL divergence’s zero-forcing effect prioritizes high-confidence predictions, effectively mitigating the influence of unreliable weak supervision. Theoretically, we extend existing bounds and derive tighter lower bounds for both forward and reverse KL divergence. Notably, when a sufficiently pre-trained strong model is fine-tuned on the last linear layer, reverse KL guarantees that it outperforms its weak supervisor by the magnitude of their disagreement. Empirically, we demonstrate that reverse KL and reverse cross-entropy not only enable strong models to outperform those trained with forward KL and standard cross-entropy across most settings, but also exhibit greater robustness to noisy labels.
\end{abstract}

\section{Introduction}

Human supervision is indispensable to align Large Language Models (LLMs) with human values~\citep{bai2022training,achiam2023gpt}.
However, as LLMs approach superhuman capabilities, their behaviors may exceed human ability to reliably manage~\citep{openai_superalignment}. 
To address this challenge, 
Weak-to-Strong Generalization (W2SG)~\citep{burns2023weak} emerges as a promising approach, leveraging weaker models to guide and control more advanced systems, thereby bridging the gap between human oversight and superhuman AI capabilities.

In particular, W2SG demonstrates that strong pre-trained LLMs, when fine-tuned under weak model supervision, can achieve performance surpassing that of their weak supervisors. 
However, this approach is fundamentally constrained by the inherent imperfections of weak model supervision, which may introduce inaccuracies and noise~\citep{burns2023weak}. 
Blindly fitting the strong model to these imperfect signals can lead to a significant discrepancy between the ground truth and the model's predictions, ultimately undermining the effectiveness of W2SG~\citep{yao2025understanding}. 
This raises a critical question: \textit{How to effectively leverage weak supervision to guide strong models while mitigating the impact of noisy or inaccurate signals?}

\begin{figure*}[t]
\begin{center}
\subfigure[Knowledge Distillation]{ 
\begin{minipage}[t]{0.48\linewidth}  \centerline{\includegraphics[width=1\linewidth]{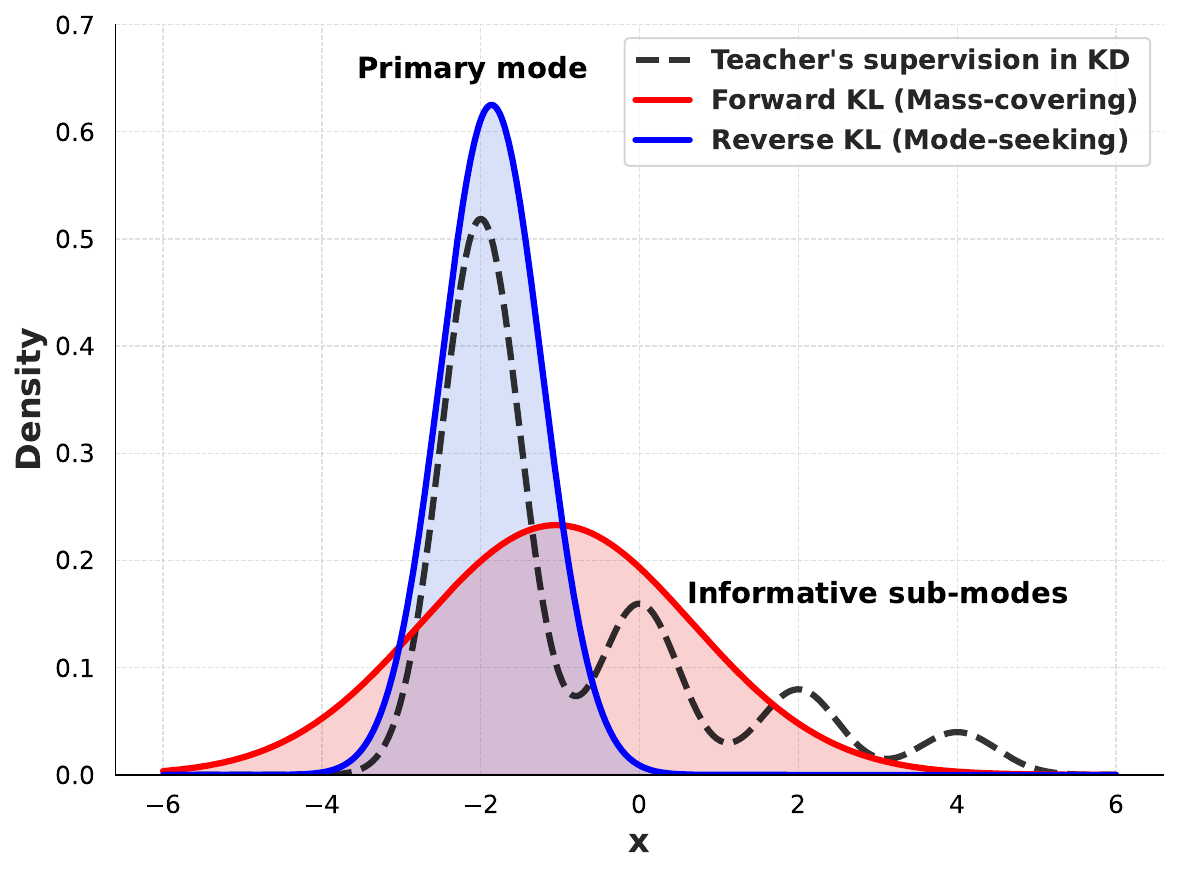}}
\end{minipage}  
}  
\subfigure[Weak-to-Strong Generalization]{
\begin{minipage}[t]{0.48\linewidth}
\centerline{\includegraphics[width=1\linewidth]{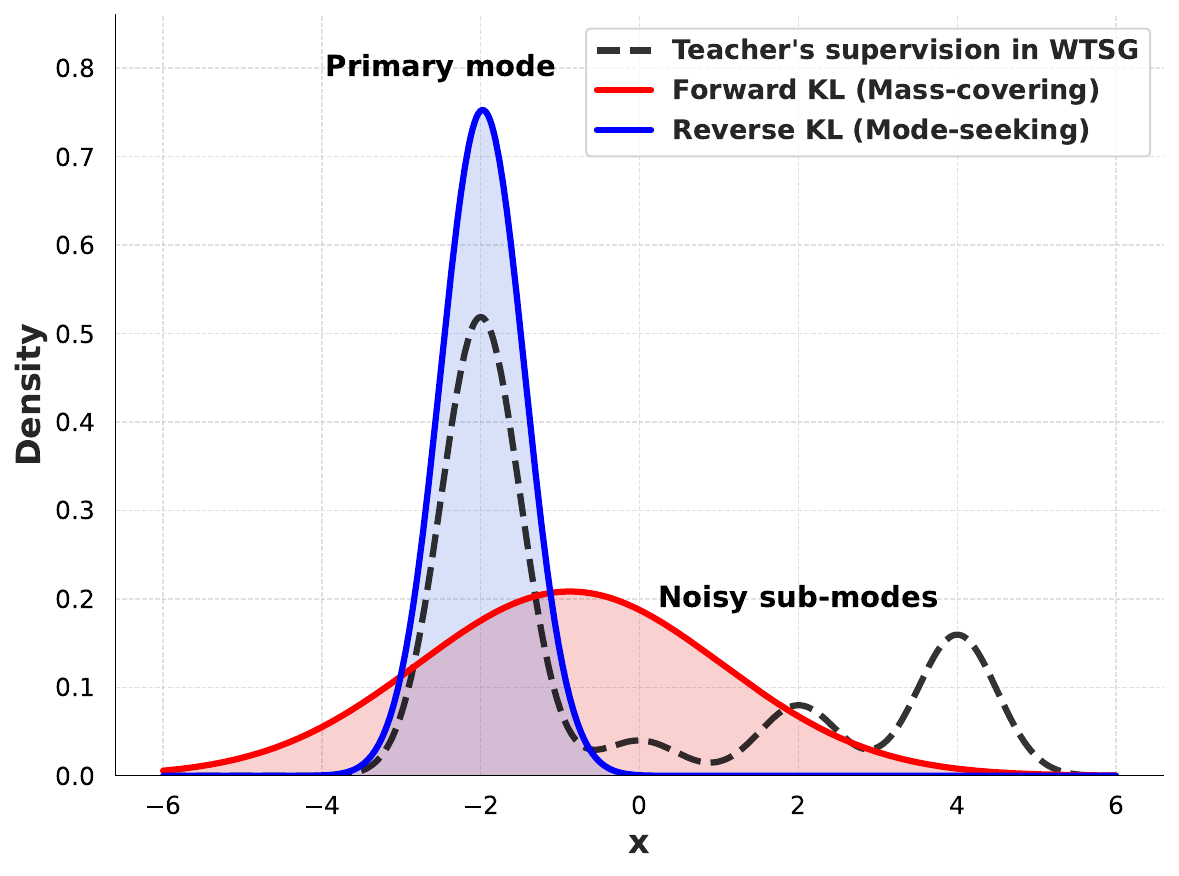}}
\end{minipage} 
}
\vspace{-5pt}
\caption{Illustration of the mass-covering behavior of forward KL divergence and the mode-seeking behavior of reverse KL divergence, highlighting their roles in KD and W2SG. A Gaussian mixture distribution, representing the teacher's supervision in KD and W2SG, is approximated by fitting a single Gaussian distribution using both forward and reverse KL divergence as loss functions.}
\label{fig:comparison_kl}
\end{center}
\vspace{-15pt}
\end{figure*}

To answer this question, we propose a theoretically principled approach, supported by fine-grained analysis and a simple yet effective solution.
Our motivation stems from an insightful comparison with Knowledge Distillation (KD)~\citep{hinton2015distilling} in classification, where strong teachers provide informative soft labels to guide weak students. 
In KD, the forward KL divergence loss plays a crucial role as it encourages students to learn not only the target class probabilities but also the relative relationships among non-target classes encoded in the teacher's soft labels.
For instance, in the image classification scenario, a strong teacher might assign higher probabilities to ``tiger'' than to ``dog'' when the input image is a ``cat'', reflecting the semantic similarity between cats and tigers in the feature space.
However, this advantageous property of forward KL in KD becomes a limitation in the W2SG paradigm.
The fundamental distinction lies in the quality of supervision: while strong teachers in KD provide reliable and informative soft labels, weak teachers in W2SG often generate noisy and potentially misleading signals for non-target classes~\citep{burns2023weak}.
Thus, the mass-covering nature of forward KL~\citep{jerfel2021variational,sun2024inverse}, which forces the student to match the entire probability distribution of the teacher's predictions, becomes detrimental in W2SG as it may lead the strong model to overfit to the weak teacher's unreliable supervision. This observation motivates our investigation of reverse KL divergence as a more suitable alternative for W2SG.
As shown in~\cref{fig:comparison_kl}, the key advantage of reverse KL lies in its mode-seeking behavior~\citep{minka2005divergence,ji2023language}, which enables the strong model to focus on the weak teacher's high-confidence predictions while being less sensitive to potentially noisy low-probability regions.
This property aligns better with the W2SG setting, as it allows the strong model to extract reliable patterns from weak supervision without being overly constrained by its imperfections. 

Building on the intuitive motivation above, we first conduct a theoretical analysis to compare forward losses and reverse losses in the context of W2SG. 
Inspired by the lower and upper bounds established for the strong model in W2SG~\citep{yao2025understanding}, we extend these results and derive better lower bounds for both forward and reverse losses.
Furthermore, we identify an advantage of reverse KL: when an adequately pre-trained strong model undergoes last linear layer fine-tuning, reverse KL guarantees that the strong student will outperform its weak teacher by at least the magnitude of their disagreement.
In our experiments, we empirically demonstrate that employing reverse KL divergence and reverse Cross-Entropy (CE) as loss functions enables the strong model to achieve superior performance compared to using forward KL divergence and standard CE.
We further show that these reverse losses are more robust to label noise.
Finally, we extend the analysis to an improved algorithm discussed in~\citet{burns2023weak}, where the optimization objective incorporates an additional regularization term. It further demonstrates the practical advantages of reverse CE over standard CE in the context of W2SG.

\section{Related Work}

\paragraph{Weak-to-Strong Generalization.}
The weak-to-strong paradigm~\citep{burns2023weak} emerges as a promising framework to address the challenges of AI alignment, particularly in the context of superalignment~\citep{openai_superalignment}—where future AI systems may surpass human capabilities, rendering human supervision weak or insufficient. 
It leverages weaker models to guide stronger models, potentially unlocking their full capabilities while maintaining alignment with human values.
It has been extensively studied through algorithms~\citep{zhu2024weak,agrawal2024ensemw2s,sang2024improving,guo2024improving,lyu2024macpo,cui2024bayesian,ye2025iterative}, empirical analyses~\citep{yang2024super,ye2024weak,zhou2025weak}, and theoretical frameworks~\citep{lang2024theoretical,somerstep2024statistical,wu2024provable,charikar2024quantifying,yao2025understanding,dong2025discrepancies,medvedev2025weak,shin2024weak,xue2025representations,ildiz2025highdimensional,somerstep2025transfer}, these works primarily focus on W2SG with forward KL divergence and CE losses. 
However, the potential of reverse KL and reverse CE losses in classification under the W2SG framework has not been sufficiently explored.
This gap motivates us to systematically investigated the benefits of these reverse losses in W2SG with both theoretical insights and empirical observations.
\footnotetext{Recently, we discovered that concurrent, independent efforts~\citep{mulgund2025relating} have also theoretically explored reverse losses in classification within the W2SG framework via information geometry and convex analysis. Due to space limitations, a more detailed comparison between our approach and theirs in~\cref{discussion:concurrent}.}

\paragraph{Forward KL and Reverse KL.}
Forward KL and Reverse KL are employed in distinct applications, each offering unique advantages.
\textit{Forward KL} is widely utilized in standard classification tasks~\citep{goodfellow2016deep}, often appearing in the form of CE loss to align predicted and true label distributions.
Its mass-covering behavior~\citep{jerfel2021variational,sun2024inverse} ensures that the model comprehensively captures all high-probability regions of the target distribution, making it particularly effective in knowledge distillation~\citep{hinton2015distilling} for classification tasks.
In such tasks, the teacher model's soft labels provide informative guidance, enabling the student model to learn a representative distribution~\citep{yang-etal-2025-distilling}.
In contrast, \textit{reverse KL} is frequently adopted in variational inference~\citep{kingma2014auto,pinheiro2021variational}, where it exhibits zero-forcing behavior~\citep{minka2005divergence}.
By focusing on high-confidence predictions while disregarding low-probability regions, reverse KL prioritizes precision over diversity.
In the context of W2SG, the choice of divergence is especially significant.
Weak teachers in W2SG provide imperfect supervision signals~\citep{burns2023weak,yang2024super,yao2025understanding}, and using forward KL divergence as the loss function may lead to overfitting to these noisy or incomplete guidance.
Reverse KL, on the other hand, allows the strong model to extract reliable patterns from weak supervision without being overly constrained by its imperfections.
This property aligns well with the goal of W2SG, where the focus is on leveraging weak supervision while avoiding its pitfalls.

Furthermore, reverse KL divergence has recently gained increasing attention in related fields such as domain adaptation~\citep{nguyen2021kl} and KL-regularized reinforcement learning~\citep{rafailov2024direct,wang2023beyond,ji2024towards}.
These applications share a conceptual similarity with W2SG, as they all involve transferring knowledge across domains or models under imperfect or constrained conditions.
Moreover, beyond classification tasks, reverse KL divergence has been increasingly utilized in generation tasks within knowledge distillation~\citep{gu2024minillm,agarwal2024policy,wu2024rethinking}, owing to its mode-seeking properties.
Given these developments, it is natural to investigate the role of reverse KL in classification within the W2SG framework.

\section{Preliminaries}

\subsection{Classification}
We consider $k$-classification tasks.
Given the data domain $\cX \subseteq \R^d$ and output domain $\cY \subseteq \R^k$, let the model space be $\cF: \cX \to \cY$. 
Consider the model's outputs form a valid probability distribution, i.e., $\forall y = (y_1, \cdots, y_k)^T \in \cY$, there holds $\sum_{i=1}^k y_i=1$ and $0 < y_i \le 1$.
The forward and reverse KL divergence losses are defined below.
\begin{definition}[KL divergence losses]
Given the data distribution $\cP$ and two models $g,h \in \cF$, the \textbf{forward} KL divergence loss is defined as:
\begin{align*} 
\kl(g,h) & \triangleq \bE_{x \sim \cP} \left[ \mathrm{D}_{\mathrm{KL}}(g(x) \| h(x)) \right], \\ 
& = \bE_{x \sim \cP} \left[ \sum_{i=1}^k [g(x)]_i \log \frac{[g(x)]_i}{[h(x)]_i} \right],
\end{align*}
where $[g(x)]_i, [h(x)]_i$ represent the $i$-th elements of $g(x), h(x)$, respectively. 
Thus, the \textbf{reverse} KL divergence loss is $\kl(h,g)$.
\end{definition}

As illustrated in~\cref{fig:comparison_kl}, forward KL promotes full coverage of the target distribution, whereas reverse KL focuses on capturing the dominant mode.
Additionally, the difference between KL divergence and CE is an entropy term:
\begin{definition}[Cross-entropy losses]
Given the data distribution $\cP$ and two models $g,h \in \cF$, define the \textbf{forward} cross-entropy divergence loss:
\begin{align*} 
\cross(g,h) & \triangleq -\bE_{x \sim \cP} \left[ \sum_{i=1}^k [g(x)]_i \log [h(x)]_i \right] \\ & = \kl(g,h) + \bE_{x \sim \cP} \; H(g(x)),
\end{align*}
where $H(\cdot)$ is the Shannon entropy.
Thus, the \textbf{reverse} cross-entropy loss is $\cross(h,g)$.
\end{definition}

Consequently, note that when minimizing forward losses, the model $g$ is fixed to provide supervision signals. Thus, minimizing forward KL divergence loss is equivalent to minimizing standard CE loss as $\bE_{x \sim \cP} \; H(g(x))$ is a constant.

\subsection{Weak-to-Strong Generalization}

Consider W2SG in the context of $k$-classification tasks. We focus on the fine-tuning phase after pre-training.
The labeling function $F^\star$ maps data $x$ to its label $F^\star(x)$.
The strong model aims to learn $F_{sw} = f \circ h_s$, where $h_s$ is a fixed strong model representation and $f \in \cF_{s}$ is a task-specific function from a hypothesis class $\cF_{s}$.
In the convention setting of AI alignment~\citep{ouyang2022training}, the model is fine-tuned through ground truth data:
\begin{align}
    \label{eqn:alignment-population-minimizer}
    f_{s} = \argmin_{f \in \cF_{s}}\; \dist(F^\star, f \circ h_s),
\end{align}
where the loss $\dist(\cdot, \cdot)$ can be $\kl(\cdot, \cdot)$ or $\cross(\cdot, \cdot)$.
However, 
it is humans who provide weak supervision in the super-alignment scenario~\citep{openai_superalignment}. 
To explore this, the W2SG framework~\citep{burns2023weak} leverages a weak model’s predictions to supervise the strong model:
\begin{align}
    \label{eqn:fsw-population-minimizer}
    f_{sw} = \argmin_{f \in \cF_{s}}\; \dist(F_w, f \circ h_s),
\end{align}
where $F_w$ is a given weak model, and $\dist(\cdot, \cdot)$ is originally the standard CE loss.
If we employ reverse losses, the objective transforms into
\begin{align} \label{eqn:rkl-minimizer}
    f_{sw}^{r} = \argmin_{f \in \cF_{s}}\; \dist(f \circ h_s, F_w).
\end{align}
Regardless of the choice of loss function, the core objective is replacing ground truth data with weak supervision.
Thus, while minimizing forward losses $\dist(F_w, F_{sw})$ or reverse losses $\dist(F_{sw}, F_w)$, we simultaneously strive to achieve an $F_{sw}$ with a small generalization error $\dist(F^\star, F_{sw})$.

\section{Theoretical Analysis: Justifying Reverse KL in W2SG}

In Sections~\ref{sec:universal}, we establish that both reverse and forward losses offer comparable generalization guarantees for the strong model, indicating that \textit{reverse losses is at least as favorable as forward losses in terms of theoretical properties}.
Furthermore, our analysis in Section~\ref{sec:upper} uncovers an advantage of reverse KL divergence loss: \textit{with reverse KL loss employed in W2SG, the strong model is theoretically guaranteed to outperform the weak model} by at least the magnitude of their disagreement under some assumptions.

\subsection{Generalization Analysis of Both Losses} \label{sec:universal}

We establish that both reverse and forward losses yield comparable generalization guarantees by deriving upper and lower bounds for their respective generalization errors.
We begin with a universal result for both forward and reverse losses.

\paragraph{Upper and lower bounds.}
We extend~\citet{yao2025understanding} and establish bounds of strong model's performance.
Unlike most previous work that focuses only on forward KL and CE loss, we comprehensively examine all four loss variants: forward KL, reverse KL, forward CE, and reverse CE.

\begin{lemma}[Proved in \cref{proof_lemma_inf}] \label{lemma:upper_lower_inf}
Let $\dist(\cdot, \cdot)$ be $\kl(\cdot, \cdot)$ or $\cross(\cdot, \cdot)$.
Given the data domain $\cX$, output domain $\cY$ and models $F_w, F^\star$ defined above. 
For any strong model $F_{sw}$, there holds
\begin{align*}
     \left| \dist(F^\star,F_w) -\dist(F^\star, F_{sw}) \right| \le C_1 \sqrt{d(F_w, F_{sw})} . 
\end{align*}
where $C_1$ is a positive constant, $d(F_w, F_{sw})$ can be $\kl(F_w, F_{sw})$ or $\kl(F_{sw}, F_w)$, and $\dist(F^\star,F_{sw})$ and $\dist(F^\star,F_w)$ represent the error of strong model and weak model, respectively.
\end{lemma}

Note that $d(F_w, F_{sw})$ captures the disagreement between the strong and weak models, which serves as the minimization objective in W2SG. 
\cref{lemma:upper_lower_inf} quantifies the difference between the weak and strong models' performance from two perspectives: a lower bound and an upper bound, which is similar to~\citet{yao2025understanding}.
The \textbf{lower bound} indicates that strong model's performance cannot be arbitrarily improved using weak supervision. 
Improving the strong model depends critically on ensuring $\dist\left( F^\star, F_w \right)$ is small, underscoring \textit{the importance of weak model's performance}.
Also, whether we choose forward or reverse loss, the student-supervisor disagreement $d(F_w, F_{sw})$ is minimized. 
While reducing $\dist\left( F^\star, F_{sw} \right)$ requires increasing $d(F_w, F_{sw})$, the lower bound also implies that strong model's performance gain may be inherently constrained by W2SG's own optimization objective~\citep{yao2025understanding}.
In other words, \textit{achieving the minimal optimization objective limits the strong model’s ability} to significantly outperform its weak supervisor.
The \textbf{upper bound} ensures that strong model's error $\dist(F^\star, F_{sw})$ remains bounded and do not be arbitrarily large.
It shows that a better weak model is also crucial to improve strong model's performance.
Building on these results, we further conduct a fine-grained analysis to investigate how to achieve tighter lower and upper bounds.

\paragraph{Tighter lower bound.}
Consider the lower bound in~\cref{lemma:upper_lower_inf}, we employ alternative proof techniques rooted in information-theoretic inequalities to derive a tighter lower bound.

\begin{theorem}[Proved in~\cref{constant:theorem}] \label{theorem:residue}
Let $\dist(\cdot, \cdot)$ be $\kl(\cdot, \cdot)$ or $\cross(\cdot, \cdot)$.
Given $F_{sw}, F_w, F^\star$, then
\begin{align*}
    & \dist(F^\star, F_{sw}) \ge \dist(F^\star, F_w) - C_2 \sqrt{d(F_w, F_{sw})},
\end{align*}
where $C_2$ is a positive constant, and $d(F_w, F_{sw})$ can be $\kl(F_w, F_{sw})$ or $\kl(F_{sw}, F_w)$.
\end{theorem}

\begin{remark}
$C_2$ is generally smaller than $C_1$, leading to a tighter lower bound than~\cref{lemma:upper_lower_inf}.
\end{remark}

Similar to~\cref{lemma:upper_lower_inf}, it also highlights the importance of selecting a well-generalizing weak model and cautious optimization of the strong model to prevent overfitting to weak supervision.
Note that~\cref{theorem:residue} applies to both forward and reverse losses, which share the same theoretical properties.

\paragraph{Tighter upper bound.}
In~\cref{lemma:upper_lower_inf},
there is no theoretical guarantee that the strong model will necessarily surpass the performance of its weak supervisor in W2SG, such as $\dist(F^\star, F_{sw}) \le \dist\left( F^\star, F_w \right)$.
This raises the question of whether a tighter upper bound can be derived.
Therefore, we first explore how to achieve this goal.

\begin{proposition}[Proved in~\cref{proof:general_equation}] \label{prop:general_equation}
Let $\dist(\cdot, \cdot)$ be $\kl(\cdot, \cdot)$ or $\cross(\cdot, \cdot)$. Given $F_{sw}, F_w, F^\star$, then there holds
\begin{align*}
    \dist(F^\star, F_{sw}) = \dist(F^\star, F_w) - \underbrace{\left \langle F^\star, \log{\frac{F_{sw}}{F_w}} \right \rangle_E}_{R},
\end{align*}
where the expectation inner product is defined as $\left \langle f,g \right \rangle_E \triangleq \bE_{x \sim \cP} [f(x)^Tg(x)]$.
\end{proposition}

\begin{remark}
    It can also be extended to reverse KL and squared loss, as detailed in~\cref{proof:general_equation}.
\end{remark}

Therefore, $\dist(F^\star, F_{sw}) \le \dist\left( F^\star, F_w \right)$ satisfies \textit{if and only if} $R \ge 0$.
To achieve it, we aim to establish a clear relationship between model capacity and model confidence across all data points and all $k$ classes.
Specifically, for any $x \in \cX$ and $i \in \{ 1, \cdots, k  \}$, a positive $[F^\star(x)]_i \log{\frac{[F_{sw}(x)]_i}{[F_w(x)]_i}}$ ensures a positive $R$.
Therefore, we expect the model predictions to satisfy either of the two inequalities:
\begin{align}
    & [F^\star(x)]_i \ge [F_{sw}(x)]_i \ge [F_w(x)]_i, \label{ineq:1} \\ 
    & [F^\star(x)]_i \le [F_{sw}(x)]_i \le [F_w(x)]_i. \label{ineq:2}
\end{align}
In other words, the predicted probabilities of $F_{sw}$ reflect the true outcome better than $F_w$.
Intuitively, because the weak model is pre-trained and fine-tuned on ground truth data, we can trust its decisions for major classes.
As shown in~\cref{fig:comparison_kl}, reverse KL’s mode-seeking behavior encourages the strong model to focus on the weak model’s high-confidence predictions, while disregarding low-probability, potentially noisy regions. 
This behavior facilitates the fulfillment of Inequality~\eqref{ineq:1}-\eqref{ineq:2}. 
In contrast, forward KL, with its mass-covering nature, forces the strong model to match the entire probability distribution, including unreliable signals from the weak model’s lower-probability classes, thereby hindering the fulfillment of the above inequalities. 
In the context of W2SG, where weak supervision is imperfect, reverse KL’s focus on high-confidence decisions provides stronger guarantees for strong model's performance. 
In particular, the theoretical analysis in the following section further supports this, demonstrating that reverse KL can theoretically ensure superior performance for the strong model in certain settings.


\subsection{Theoretical Analysis of Reverse Losses} \label{sec:upper}

To achieve a tighter upper bound, our theoretical analysis below yields an intriguing insight: \textit{when using reverse KL in W2SG, an adequate pre-training and subsequent last linear layer fine-tuning guarantee that the strong student can outperform its weak teacher} by at least the magnitude of their disagreement (i.e., $R \ge 0$ in~\cref{prop:general_equation}).

\begin{theorem}[Proved in \cref{theorem1_kl_loss}]
\label{thm:realizable}
Consider W2SG using reverse KL divergence loss in~\cref{eqn:rkl-minimizer}.
Denote $F_{sw}=f^r_{sw} \circ h_s$.
Assume that the function class $\cF_{s}$ is a convex set and $\exists f_s \in \cF_s$ such that $f_s \circ h_s = F^\star$.
Then:
\begin{align*}
    \kl\left(F^\star, F_{sw}\right) \leq \kl\left(F^\star, F_w\right)-\kl\left(F_{sw}, F_w\right).
\end{align*}
\end{theorem}

\begin{remark}
Similar result can be naturally extended to reverse CE loss.
Our proof leverages Bregman divergence, a generalization of both squared loss and KL divergence. 
This approach not only broadens the applicability of our results but also demonstrates how our framework naturally recovers the regression analysis~\citep{charikar2024quantifying}.
The concurrent work~\citep{mulgund2025relating} also independently explores the application of Bregman divergence in this context, establishing their Theorem 4.1 and Corollary 4.2, which exhibit parallels with our results.
The above extension and discussion are detailed in~\cref{theorem1_kl_loss}.
\end{remark}

The assumptions are consistent with previous theory~\citep{charikar2024quantifying,yao2025understanding}.
Firstly, the convex set assumption includes the case that $\cF_s$ is the class of all linear functions, which shares similar conceptual framework of last linear layer fine-tuning~\citep{howard2018universal,kumar2022fine,mao2023last,kirichenko2023last}.
Secondly, we consider the case where $\exists f_s \in \cF_s$ such that $f_s \circ h_s = F^\star$.
It shows the remarkable capability of pre-training. 
It assumes that the representation $h_s$ has already captured a wealth of information during pre-training, a phenomenon well-demonstrated by modern pre-trained LLMs~\citep{touvron2023llama,achiam2023gpt}.

\cref{thm:realizable} establishes that in W2SG, using the reverse KL divergence loss guarantees that the strong model, trained with weak supervision, surpasses the weak model by at least their disagreement, $\kl(F_{sw}, F_w)$.
This upper bound is tighter than~\cref{lemma:upper_lower_inf}, as~\cref{lemma:upper_lower_inf} does not ensure that the strong model surpasses the weak model.
\cref{thm:realizable} highlights the superior theoretical benefits of reverse losses compared to forward losses.

Now we draw $n$ i.i.d. samples to perform W2SG and relax the assumption, where for any $f_s \in \cF_s$, $\exists f_s \circ h_s = F^\star$ may not be satisfied.
The unique result for reverse KL below further emphasizes its advantageous theoretical properties in W2SG.

\begin{theorem}[Proved in~\cref{proof_non-realizable}] \label{thm:non-realizable-finite-samples}
Given $F_{sw}$ defined in~\cref{thm:realizable}.
Assume that $\cF_s$ is a convex set.
Consider W2SG using reverse KL divergence loss with $n$ i.i.d. samples:
\begin{align*}
    \hat{f}^r_{sw} = \argmin_{f \in \cF_{s}}\; \widehat{\kl}(f \circ h_s, F_w),
\end{align*}
where $\widehat{\kl}(\cdot, \cdot)$
is the empirical version of $\kl(\cdot, \cdot)$.
Denote $\hat{F}_{sw}=\hat{f}^r_{sw} \circ h_s$ and strong ceiling model's generalization error $\eps = \kl(F^\star, F_s)$.
With probability at least $1-\delta$, there holds
\begin{align*} 
& \kl(F^\star, \hat{F}_{sw}) \le \kl(F^\star, F_w) - \kl(\hat{F}_{sw}, F_w) \\ & \hspace{0.05cm} + \cO(\sqrt{\eps}) +  \cO\left(\sqrt{\frac{\cC_{\cF_s}}{n}}\right) + \cO\left(\sqrt{\frac{\log(1/\delta)}{n}}\right),
\end{align*}
where $\cC_{\cF_s}$ is a constant capturing the complexity of the function class $\cF_s$.
The asymptotic notation is for $\eps \to 0, n \to \infty$.
\end{theorem}

Compared to~\cref{thm:realizable}, this bound introduces two additional error terms: the first term $\cO(\sqrt{\eps})$ is small due to the capability of the strong ceiling model $F_s$.
The remaining two error terms, which are of the order $\cO \left( 1/\sqrt{n} \right)$, stem from the strong model $\hat{F}_{sw}$ being trained on a finite weakly-labeled dataset. 
These terms also diminish asymptotically as the sample size $n$ increases.
Overall, by using a sufficiently large dataset and a strong model with enough capacity, we achieve a large $n$ and a very small $\eps$, rendering the remainders in \cref{thm:non-realizable-finite-samples} negligible and increasing the likelihood that the theoretical guarantee in \cref{thm:realizable} holds.
\cref{thm:non-realizable-finite-samples} aligns with previous wisdom~\citep{charikar2024quantifying,yao2025understanding}. 
However, whereas their corresponding bounds are specifically designed for regression tasks, our result offers new insights into applying reverse KL loss in classification tasks.

\section{Empirical Validation}

\begin{figure*}[t]
\begin{center}
\subfigure[Results of GPT-2-series on CAI-Harmless]{ 
\begin{minipage}[t]{0.95\linewidth}  \centerline{\includegraphics[width=1\linewidth]{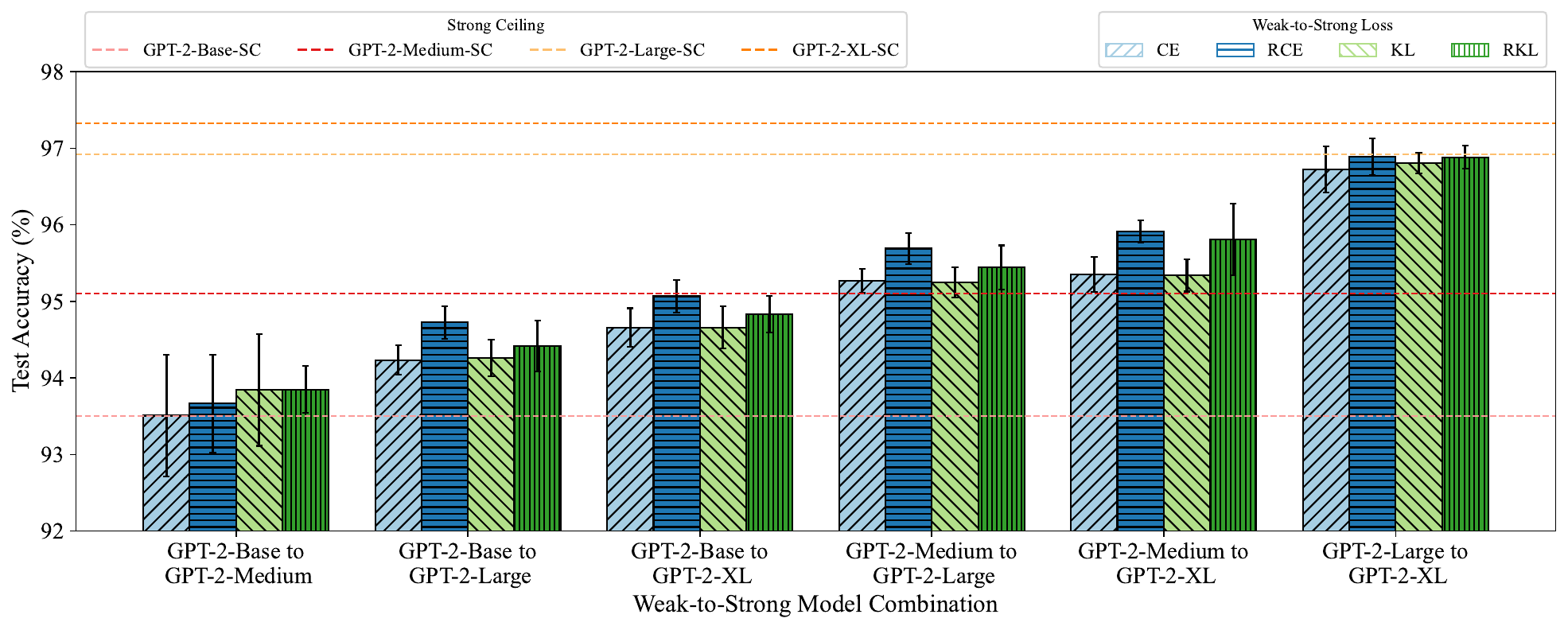}}
\end{minipage}
}
\vspace{-3pt}
\subfigure[Results of GPT-2-series on helpful set of HH-RLHF]{
\begin{minipage}[t]{0.95\linewidth}
\centerline{\includegraphics[width=1\linewidth]{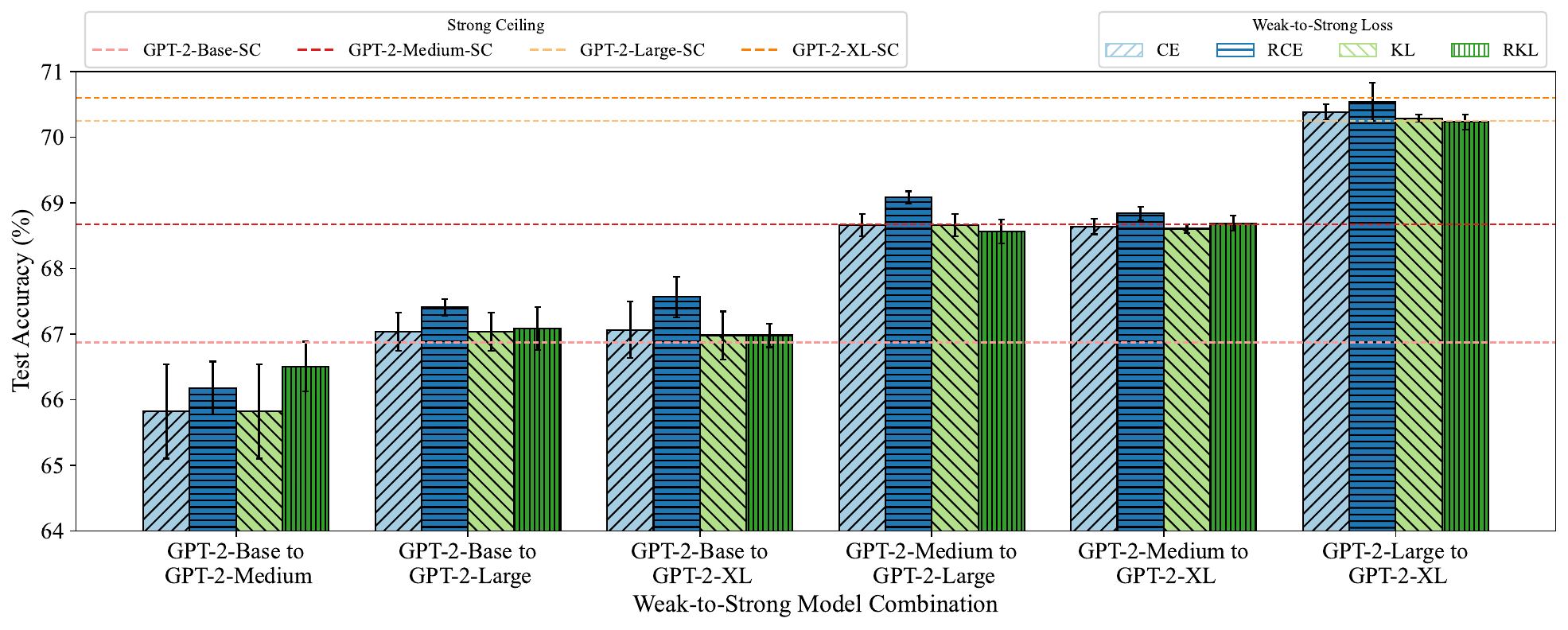}}
\end{minipage}  
}    
\vspace{-7pt}
\caption{Results of GPT-2-series. ``SC'' denotes the strong ceiling model, and ``A to B'' indicates the use of weak teacher ``A'' to supervise strong student ``B''. The terms CE, RCE, KL, and RKL refer to CE loss, reverse CE loss, forward KL divergence loss, and reverse KL divergence loss, respectively. Error bars represent the standard deviation across three runs of the experiment.}
\label{fig:cai}
\end{center}
\vspace{-15pt}
\end{figure*}

In this section, we empirically compare reverse KL, forward KL, reverse CE, and standard CE losses in the context of W2SG. 
Our experiments directly support the claim that reverse losses outperform forward losses in most experimental settings.

\begin{table*}[t]
\centering
\begin{tabular}{cccccc}
\toprule
\textbf{Models}                     & \textbf{CE}                 & \textbf{Reverse CE}         & \textbf{KL}                 & \textbf{Reverse KL}         \\ 
\midrule
3B $\rightarrow$ 7B  & 95.000 $\pm$ 0.177   & \textbf{95.358 $\pm$ 0.225} & 95.067 $\pm$ 0.250  & \underline{95.283 $\pm$ 0.153} \\ 
3B $\rightarrow$ 14B & 95.900 $\pm$ 0.001   & \textbf{96.100 $\pm$ 0.001}   & 95.700 $\pm$ 0.003   & \underline{95.800 $\pm$ 0.003}   \\ 
7B $\rightarrow$ 14B & 96.117 $\pm$ 0.003 & \textbf{96.817 $\pm$ 0.001} & 96.441 $\pm$ 0.003 & \underline{96.508 $\pm$ 0.001} \\
\bottomrule
\end{tabular}
\caption{Results on Qwen2.5 series. ``3B $\rightarrow$ 7B'' means using Qwen2.5-3B as the weak teacher to supervise the strong student Qwen2.5-7B. We report the results with three runs of the experiments. The best and the second-best results are highlighted in \textbf{bold} and \underline{underlined} text, respectively.}
\label{W2SG:qwen}
\end{table*}

\begin{table*}[t]
\centering
\begin{tabular}{ccccccc}
\toprule
\textbf{Setting} & \textbf{Loss} & \textbf{10\%} & \textbf{20\%} & \textbf{30\%} & \textbf{40\%} & \textbf{50\%} \\
\midrule
\multirow{4}{*}{GPT-2-Base $\rightarrow$ GPT-2-Medium} 
& KL     & 90.050  & 86.250  & 81.650   & 72.800    & \underline{53.950}   \\
& RKL    & \textbf{92.400}   & \textbf{91.300}   & \textbf{90.025}  & \underline{80.625}  & 45.375  \\
& CE     & 90.100   & 86.200   & 81.625  & 72.800    & 53.950   \\
& RCE    & \underline{92.025} & \underline{90.800}   & \underline{89.450}   & \textbf{81.825}  & \textbf{59.725}  \\
\cmidrule(lr){1-7}

\multirow{4}{*}{GPT-2-Base $\rightarrow$ GPT-2-Large} 
& KL     & \textbf{94.000}    & 92.325 & 91.275 & 84.750  & 30.550  \\
& RKL    & 93.725 & \textbf{94.150}  & \textbf{92.150}  & \textbf{91.350}  & \textbf{35.175} \\
& CE     & \underline{94.000}    & 92.325 & 91.275 & 84.725 & \underline{30.550}  \\
& RCE    & 93.875 & \underline{93.650}  & \underline{91.900}   & \underline{85.875} & 26.975 \\
\cmidrule(lr){1-7}

\multirow{4}{*}{GPT-2-Medium $\rightarrow$ GPT-2-Large} 
& KL     & \underline{93.800} & 92.850 & 91.730 & 85.930 & \underline{30.000} \\ 
& RKL    & \textbf{94.000} & \underline{93.850} & \underline{92.930} & \textbf{88.730} & 27.400 \\
& CE     & \underline{93.800} & 92.850 & 91.725 & 85.925 & 30.000 \\
& RCE    & 93.675 & \textbf{94.150} & \textbf{93.025} & \underline{88.550} & \textbf{33.775} \\
\bottomrule
\end{tabular}
\caption{Performance comparison across different noise levels and model settings. Bold numbers indicate the best performance for each noise level within each setting. For each experimental setting and noise level, the top-performing result is highlighted in \textbf{bold}, while the second-best is indicated with an \underline{underline}.}
\label{tab:noise_comparison}
\vspace{-10pt}
\end{table*}

\subsection{Experimental Settings}

\paragraph{Datasets.}
We follow previous studies~\citep{burns2023weak,yang2024super} to conduct experiments mainly in the reward modeling task in two settings: enabling a weak model to effectively guide a strong model in achieving either harmlessness or helpfulness. 
To achieve \textbf{harmlessness}, we follow~\citet{yang2024super} to leverage CAI-Harmless~\citep{bai2022constitutional}, a widely adopted benchmark for single-turn harmless dialogue tasks. 
To achieve \textbf{helpfulness}, we utilize HH-RLHF~\citep{bai2022training}, a benchmark designed to guide models toward producing responses that are helpful, informative, and contextually relevant. 
We use a subset of single-turn helpful data of HH-RLHF.

Each dataset includes three subsets: \textbf{(1) Ground truth set}: 4K samples with ground truth labels, used to fine-tune the base models to create strong ceiling models.
\textbf{(2) Weak supervision set}: 4K held-out samples, where the weak model generates predicted labels to guide the training of the strong model.
\textbf{(3) Test set}: The extra 4K samples, reserved for evaluating the generalization performance of all strong ceiling and weak-to-strong models. 
Each sample is formatted as $\Tilde{x}=(x;y_c,y_r)$, where $x$ is the user input, $y_c$ and $y_r$ represent human-chosen and human-rejected responses separately.

\paragraph{Models.}
We conduct experiments on two types of model families: 
(1) GPT-2-series~\citep{radford2019language}, including GPT-2-Base, GPT-2-Medium, GPT-2-Large, and GPT-2-XL; 
(2) Pythia-series~\citep{biderman2023pythia}, specifically, Pythia-70M, Pythia-160M, Pythia-410M, and Pythia-1B. 
Each model is trained to generate a soft value between 0 to 1 for each sample: $$F(\Tilde{x}) = \text{Sigmoid}(F(y_c)-F(y_r)).$$
When implementing forward and reverse losses, the single predicted logit is transformed into a logits distribution represented as $(1 - F(\Tilde{x}), F(\Tilde{x}))$.

\paragraph{Training and Evaluation.}
The strong ceiling models are trained using the standard CE loss. We apply four loss functions in W2SG: forward KL, reverse KL, CE and reverse CE.
To ensure the reliability and consistency of our results, each experiment is repeated across three random seeds.
We set the training batch size to $16$, learning rate to $10^{-5}$, and \texttt{max\_seq\_len} to $512$.
Following the approach of~\citet{burns2023weak}, we train each model for a single epoch to reduce overfitting.
Finally, we report the average accuracy on the test set across the three random seeds for each model for comparison.

\subsection{Main Results} \label{sec:main_results}

The experimental results of the GPT-2 series on the CAI-Harmless and HH-RLHF datasets are presented in~\cref{fig:cai}. 
Due to space limitation, we put the detailed results for the Pythia series in~\cref{exp_result_pythia}, but the similar trends can be observed. 

We can draw several conclusions from the results in~\cref{fig:cai}: 
(1) The accuracy demonstrates a consistent upward trend from left to right. 
It indicates that the generalization capability of the strong model improves when a more capable weak model is employed as the supervisor.
This finding is in line with~\cref{lemma:upper_lower_inf} and aligns with prior research~\citep{burns2023weak,yao2025understanding}, which suggests that utilizing a higher-capacity weak model enhances the strong model's performance.
Furthermore, with a fixed weak model, leveraging a stronger model also yield improved strong model's performance, consistent with established research~\citep{burns2023weak,yang2024super}.
(2) We observe that \textit{reverse KL and reverse CE losses enable strong models to outperform those trained with forward KL and CE losses across most experimental settings}.
In particular, in all settings (12 out of 12), the use of reverse KL yields a stronger model compared to standard KL.
Similarly, reverse CE outperforms or parallels forward CE in nearly all experimental settings (10 out of 12). 
These empirical results, supported by our theoretical framework, underscore the superiority of reverse losses over forward losses.
(3) In the majority of settings (10 out of 12), the strong model surpasses or meets the performance of its weak supervisor when trained with reverse KL or reverse CE loss. This observation supports~\cref{thm:realizable} and~\cref{thm:non-realizable-finite-samples}. 
However, the theoretical guarantees may not always hold in practice, particularly in scenarios involving extremely complex LLMs with limited training set in W2SG, where the underlying assumptions may be violated.

\begin{figure*}[t]
\begin{center}
\vspace{-15pt}
\centerline{\includegraphics[width=0.9\linewidth]{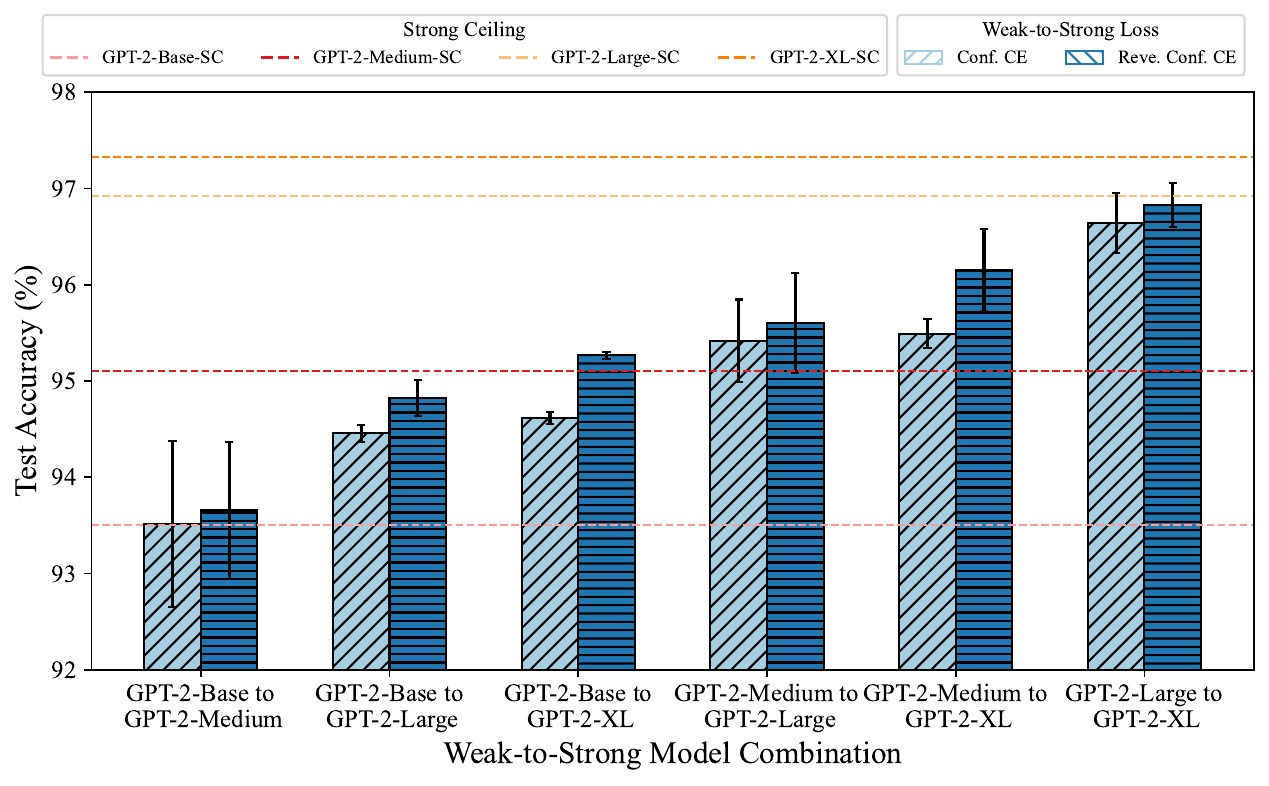}}
\vspace{-5pt}
\caption{Results of GPT-2 series on CAI-Harmless. ``SC'' denotes the strong ceiling model, and ``A to B'' indicates the use of weak teacher ``A'' to supervise strong student ``B''. The terms ``Conf.\ CE'' and ``Reve.\ Conf.\ CE'' refer to the auxiliary confidence loss with vanilla cross-entropy loss (\cref{eq:confidence_loss}) and reverse cross-entropy loss (\cref{eq:reverse_confidence_loss}), respectively. Error bars represent the standard deviation across three runs of the experiment.}
\label{fig:conf_loss}
\end{center}
\vspace{-25pt}
\end{figure*}


\subsection{Ablation Study}

\paragraph{Larger-scale language models.}
We have also conducted additional experiments on Qwen2.5-3B, Qwen2.5-7B and Qwen2.5-14B models~\citep{qwen} using CAI-Harmless dataset in~\cref{W2SG:qwen}.
We can still observe that reverse losses consistently outperform forward losses across these larger-scale models.

\paragraph{Noise tolerance.}
We introduce additional noise into the weak supervision and conduct experiments on CAI-Harmless. 
Specifically, we systematically vary the proportion of training data (ranging from 10\% to 50\%) where we swap the probability assignments between chosen and rejected responses in the weak supervision signals.
When the chosen and rejected labels are systematically swapped across a substantial portion of the weak supervision, the model's optimization process becomes misdirected, causing it to learn inverted preferences. In this case, the model's accuracy may fall below the 50\% (random guessing), as it actively learns to select the incorrect responses.
The results are shown in~\cref{tab:noise_comparison}. We perform each experiment in one run. 
Our results demonstrate that reverse losses outperform forward losses in most experimental settings. However, under extremely high noise levels, reverse KL divergence may exhibit overconfidence in incorrect modes, leading to degraded performance.

\paragraph{Regularization.}
We notice that~\citet{burns2023weak} investigates an improved strategy: incorporating an additional regularization term aimed at boosting the strong model's confidence in its predictions, while utilizing the standard CE loss as the primary objective.
This naturally raises the question of whether combining reverse CE loss with such regularization can further improve the strong model's performance compared to standard CE loss with regularization.
To explore this question, we conduct experiments using the GPT-2 series on the CAI-Harmless dataset as a representative case. 
The key observations are summarized in~\cref{fig:conf_loss}, with further experimental details deferred to~\cref{exp:conf_loss}.
First, by integrating the insights from~\cref{fig:cai} and~\cref{fig:conf_loss}, we can see that incorporating the confidence regularization leads to a modest improvement in the strong model's performance, aligning with the observations of~\citet{burns2023weak}.
Second, the strong model trained using reverse CE loss with regularization consistently surpasses its counterpart trained with standard CE loss. 
This result, together with our previous results in~\cref{sec:main_results}, underscores the clear advantage of reverse losses over forward losses in enhancing model performance in diverse settings.

\section{Conclusion}

In this work, we propose a theoretically principled approach by rethinking the loss function in W2SG.
Unlike the mass-covering nature of forward KL, reverse KL exhibits a mode-seeking behavior that focuses on high-confidence predictions from the weak supervisor, thereby reducing the influence of noisy signals.
Theoretically, we derive both upper and lower bounds for forward and reverse losses and explain how to make these bounds tighter.
Notably, when fine-tuning a pre-trained strong model on its last linear layer, reverse KL theoretically ensures that the strong model outperforms its weak supervisor by the magnitude of their disagreement under some assumptions.
Empirically, we show that reverse losses successfully outperform forward losses in most settings and exhibit better noise toleration, highlighting the practical benefits of reverse KL and CE losses in W2SG.

\newpage

\section*{Limitations}
While our study provides theoretical insights and empirical validation for the advantages of reverse losses in W2SG, several limitations remain. 
First, our analysis mainly assumes relatively reliable weak supervision from pre-trained and fine-tuned models. However, real-world applications often involve noisy weak supervision, and reverse KL’s mode-seeking nature may amplify extreme noise. Further research is needed to assess its suitability in such cases.
Second, while the theoretical results in \cref{sec:universal} provide broad insights, the assumptions in \cref{sec:upper} may not hold in the practical deployment of LLMs. 
This limitation is shared by most related work on theoretical understanding of W2SG. 
Nonetheless, these foundations offer valuable guidance and a starting point for future research on advancing W2SG theory in LLMs.
Third, our experiments are conducted on two well-known alignment-focused binary classification datasets with relatively smaller model sizes. While these results offer valuable insights, it remains an open question whether they can be generalized to more diverse datasets and larger-scale models. Exploring this aspect in future work will help further validate the broader applicability of our approach.

\section*{Ethics Statement}
Our intention is to highlight the positive impact of reverse losses in improving weak-to-strong generalization, ensuring more robust and reliable model performance while minimizing the influence of potentially imperfect weak supervision.
However, the potential amplification of biases from weak models remains a concern, particularly in sensitive applications where fairness is a critical issue. While reverse KL mitigates overfitting to unreliable supervision, its mode-seeking nature may amplify the biases present in the weak model’s predictions.
Additionally, stronger AI models trained using W2SG could be misused if deployed without appropriate safeguards, emphasizing the need for responsible development and oversight.

\section*{Acknowledgements}
We are deeply grateful to Gengze Xu, Abhijeet Mulgund and Chirag Pabbaraju for their invaluable insights and constructive suggestions.
This research was supported by National Natural Science Foundation of China (No.62476277), National Key Research and Development Program of China(NO. 2024YFE0203200), CCF-ALIMAMA TECH Kangaroo Fund(No.CCF-ALIMAMA OF 2024008), and Huawei-Renmin University joint program on Information Retrieval. We also acknowledge the support provided by the fund for building worldclass universities (disciplines) of Renmin University of China and by the funds from Beijing Key Laboratory of Big Data Management and Analysis Methods, Gaoling School of Artificial Intelligence, Renmin University of China, from Engineering Research Center of Next-Generation Intelligent Search and Recommendation, Ministry of Education, from Intelligent Social Governance Interdisciplinary Platform, Major Innovation \& Planning Interdisciplinary Platform for the “DoubleFirst Class” Initiative, Renmin University of China, from Public Policy and Decision-making Research Lab of Renmin University of China, and from Public Computing Cloud, Renmin University of China.


\bibliography{custom}

\begin{thebibliography}{52}
\providecommand{\natexlab}[1]{#1}

\bibitem[{Agarwal et~al.(2024)Agarwal, Vieillard, Zhou, Stanczyk, Garea, Geist, and Bachem}]{agarwal2024policy}
Rishabh Agarwal, Nino Vieillard, Yongchao Zhou, Piotr Stanczyk, Sabela~Ramos Garea, Matthieu Geist, and Olivier Bachem. 2024.
\newblock On-policy distillation of language models: Learning from self-generated mistakes.
\newblock In \emph{International Conference on Learning Representations}.

\bibitem[{Agrawal et~al.(2024)Agrawal, Ding, Che, Deng, Satheesh, Langford, and Huang}]{agrawal2024ensemw2s}
Aakriti Agrawal, Mucong Ding, Zora Che, Chenghao Deng, Anirudh Satheesh, John Langford, and Furong Huang. 2024.
\newblock Ensemw2s: Can an ensemble of llms be leveraged to obtain a stronger llm?
\newblock \emph{arXiv preprint arXiv:2410.04571}.

\bibitem[{Bai et~al.(2022{\natexlab{a}})Bai, Jones, Ndousse, Askell, Chen, DasSarma, Drain, Fort, Ganguli, Henighan et~al.}]{bai2022training}
Yuntao Bai, Andy Jones, Kamal Ndousse, Amanda Askell, Anna Chen, Nova DasSarma, Dawn Drain, Stanislav Fort, Deep Ganguli, Tom Henighan, et~al. 2022{\natexlab{a}}.
\newblock Training a helpful and harmless assistant with reinforcement learning from human feedback.
\newblock \emph{arXiv preprint arXiv:2204.05862}.

\bibitem[{Bai et~al.(2022{\natexlab{b}})Bai, Kadavath, Kundu, Askell, Kernion, Jones, Chen, Goldie, Mirhoseini, McKinnon et~al.}]{bai2022constitutional}
Yuntao Bai, Saurav Kadavath, Sandipan Kundu, Amanda Askell, Jackson Kernion, Andy Jones, Anna Chen, Anna Goldie, Azalia Mirhoseini, Cameron McKinnon, et~al. 2022{\natexlab{b}}.
\newblock Constitutional ai: Harmlessness from ai feedback.
\newblock \emph{arXiv preprint arXiv:2212.08073}.

\bibitem[{Biderman et~al.(2023)Biderman, Schoelkopf, Anthony, Bradley, O’Brien, Hallahan, Khan, Purohit, Prashanth, Raff et~al.}]{biderman2023pythia}
Stella Biderman, Hailey Schoelkopf, Quentin~Gregory Anthony, Herbie Bradley, Kyle O’Brien, Eric Hallahan, Mohammad~Aflah Khan, Shivanshu Purohit, USVSN~Sai Prashanth, Edward Raff, et~al. 2023.
\newblock Pythia: A suite for analyzing large language models across training and scaling.
\newblock In \emph{International Conference on Machine Learning}, pages 2397--2430.

\bibitem[{Burns et~al.(2023)Burns, Izmailov, Kirchner, Baker, Gao, Aschenbrenner, Chen, Ecoffet, Joglekar, Leike et~al.}]{burns2023weak}
Collin Burns, Pavel Izmailov, Jan~Hendrik Kirchner, Bowen Baker, Leo Gao, Leopold Aschenbrenner, Yining Chen, Adrien Ecoffet, Manas Joglekar, Jan Leike, et~al. 2023.
\newblock Weak-to-strong generalization: Eliciting strong capabilities with weak supervision.
\newblock \emph{arXiv preprint arXiv:2312.09390}.

\bibitem[{Charikar et~al.(2024)Charikar, Pabbaraju, and Shiragur}]{charikar2024quantifying}
Moses Charikar, Chirag Pabbaraju, and Kirankumar Shiragur. 2024.
\newblock Quantifying the gain in weak-to-strong generalization.
\newblock \emph{Advances in neural information processing systems}.

\bibitem[{Cui et~al.(2025)Cui, Zhang, Sun, Wu, and Zhang}]{cui2024bayesian}
Ziyun Cui, Ziyang Zhang, Guangzhi Sun, Wen Wu, and Chao Zhang. 2025.
\newblock Bayesian weaks-to-strong from text classification to generation.
\newblock In \emph{The Thirteenth International Conference on Learning Representations}.

\bibitem[{Dong et~al.(2025)Dong, Li, Li, Lee, and Lei}]{dong2025discrepancies}
Yijun Dong, Yicheng Li, Yunai Li, Jason~D Lee, and Qi~Lei. 2025.
\newblock Discrepancies are virtue: Weak-to-strong generalization through lens of intrinsic dimension.
\newblock \emph{arXiv preprint arXiv:2502.05075}.

\bibitem[{Goodfellow(2016)}]{goodfellow2016deep}
Ian Goodfellow. 2016.
\newblock \emph{Deep learning}, volume 196.
\newblock MIT press.

\bibitem[{Gu et~al.(2024)Gu, Dong, Wei, and Huang}]{gu2024minillm}
Yuxian Gu, Li~Dong, Furu Wei, and Minlie Huang. 2024.
\newblock Minillm: Knowledge distillation of large language models.
\newblock In \emph{International Conference on Learning Representations}.

\bibitem[{Guo and Yang(2024)}]{guo2024improving}
Yue Guo and Yi~Yang. 2024.
\newblock Improving weak-to-strong generalization with reliability-aware alignment.
\newblock \emph{arXiv preprint arXiv:2406.19032}.

\bibitem[{Hinton(2015)}]{hinton2015distilling}
Geoffrey Hinton. 2015.
\newblock Distilling the knowledge in a neural network.
\newblock \emph{arXiv preprint arXiv:1503.02531}.

\bibitem[{Howard and Ruder(2018)}]{howard2018universal}
Jeremy Howard and Sebastian Ruder. 2018.
\newblock Universal language model fine-tuning for text classification.
\newblock In \emph{Proceedings of the 56th Annual Meeting of the Association for Computational Linguistics (Volume 1: Long Papers)}, pages 328--339.

\bibitem[{Ildiz et~al.(2025)Ildiz, Gozeten, Taga, Mondelli, and Oymak}]{ildiz2025highdimensional}
Muhammed~Emrullah Ildiz, Halil~Alperen Gozeten, Ege~Onur Taga, Marco Mondelli, and Samet Oymak. 2025.
\newblock High-dimensional analysis of knowledge distillation: Weak-to-strong generalization and scaling laws.
\newblock In \emph{The Thirteenth International Conference on Learning Representations}.

\bibitem[{Jerfel et~al.(2021)Jerfel, Wang, Wong-Fannjiang, Heller, Ma, and Jordan}]{jerfel2021variational}
Ghassen Jerfel, Serena Wang, Clara Wong-Fannjiang, Katherine~A Heller, Yian Ma, and Michael~I Jordan. 2021.
\newblock Variational refinement for importance sampling using the forward kullback-leibler divergence.
\newblock In \emph{Uncertainty in Artificial Intelligence}, pages 1819--1829.

\bibitem[{Ji et~al.(2024{\natexlab{a}})Ji, Ke, Wang, and Huang}]{ji2023language}
Haozhe Ji, Pei Ke, Hongning Wang, and Minlie Huang. 2024{\natexlab{a}}.
\newblock Language model decoding as direct metrics optimization.
\newblock In \emph{International Conference on Learning Representations}.

\bibitem[{Ji et~al.(2024{\natexlab{b}})Ji, Lu, Niu, Ke, Wang, Zhu, Tang, and Huang}]{ji2024towards}
Haozhe Ji, Cheng Lu, Yilin Niu, Pei Ke, Hongning Wang, Jun Zhu, Jie Tang, and Minlie Huang. 2024{\natexlab{b}}.
\newblock Towards efficient and exact optimization of language model alignment.
\newblock In \emph{International Conference on Machine Learning}.

\bibitem[{Kingma and Welling(2014)}]{kingma2014auto}
Diederik~P Kingma and Max Welling. 2014.
\newblock Auto-encoding variational bayes.
\newblock In \emph{International Conference on Learning Representations}.

\bibitem[{Kirichenko et~al.(2023)Kirichenko, Izmailov, and Wilson}]{kirichenko2023last}
Polina Kirichenko, Pavel Izmailov, and Andrew~Gordon Wilson. 2023.
\newblock Last layer re-training is sufficient for robustness to spurious correlations.
\newblock In \emph{International Conference on Learning Representations}.

\bibitem[{Kumar et~al.(2022)Kumar, Raghunathan, Jones, Ma, and Liang}]{kumar2022fine}
Ananya Kumar, Aditi Raghunathan, Robbie Jones, Tengyu Ma, and Percy Liang. 2022.
\newblock Fine-tuning can distort pretrained features and underperform out-of-distribution.
\newblock In \emph{International Conference on Learning Representations}.

\bibitem[{Lang et~al.(2024)Lang, Sontag, and Vijayaraghavan}]{lang2024theoretical}
Hunter Lang, David Sontag, and Aravindan Vijayaraghavan. 2024.
\newblock Theoretical analysis of weak-to-strong generalization.
\newblock In \emph{The Thirty-eighth Annual Conference on Neural Information Processing Systems}.

\bibitem[{Lyu et~al.(2025)Lyu, Yan, Wang, Yin, Ren, de~Rijke, and Ren}]{lyu2024macpo}
Yougang Lyu, Lingyong Yan, Zihan Wang, Dawei Yin, Pengjie Ren, Maarten de~Rijke, and Zhaochun Ren. 2025.
\newblock Macpo: Weak-to-strong alignment via multi-agent contrastive preference optimization.
\newblock In \emph{The Thirteenth International Conference on Learning Representations}.

\bibitem[{Mao et~al.(2023)Mao, Deng, Yao, Ye, Kawaguchi, and Zou}]{mao2023last}
Yuzhen Mao, Zhun Deng, Huaxiu Yao, Ting Ye, Kenji Kawaguchi, and James Zou. 2023.
\newblock Last-layer fairness fine-tuning is simple and effective for neural networks.
\newblock \emph{arXiv preprint arXiv:2304.03935}.

\bibitem[{Medvedev et~al.(2025)Medvedev, Lyu, Yu, Arora, Li, and Srebro}]{medvedev2025weak}
Marko Medvedev, Kaifeng Lyu, Dingli Yu, Sanjeev Arora, Zhiyuan Li, and Nathan Srebro. 2025.
\newblock Weak-to-strong generalization even in random feature networks, provably.
\newblock \emph{arXiv preprint arXiv:2503.02877}.

\bibitem[{Minka et~al.(2005)}]{minka2005divergence}
Tom Minka et~al. 2005.
\newblock Divergence measures and message passing.
\newblock Technical report, Microsoft Research.

\bibitem[{Mulgund and Pabbaraju(2025)}]{mulgund2025relating}
Abhijeet Mulgund and Chirag Pabbaraju. 2025.
\newblock Relating misfit to gain in weak-to-strong generalization beyond the squared loss.
\newblock \emph{arXiv preprint arXiv:2501.19105}.

\bibitem[{Nguyen et~al.(2022)Nguyen, Tran, Gal, Torr, and Baydin}]{nguyen2021kl}
A~Tuan Nguyen, Toan Tran, Yarin Gal, Philip~HS Torr, and At{\i}l{\i}m~G{\"u}ne{\c{s}} Baydin. 2022.
\newblock Kl guided domain adaptation.
\newblock In \emph{International Conference on Learning Representations}.

\bibitem[{OpenAI(2023{\natexlab{a}})}]{achiam2023gpt}
OpenAI. 2023{\natexlab{a}}.
\newblock Gpt-4 technical report.
\newblock \emph{arXiv preprint arXiv:2303.08774}.

\bibitem[{OpenAI(2023{\natexlab{b}})}]{openai_superalignment}
OpenAI. 2023{\natexlab{b}}.
\newblock \href {https://openai.com/index/introducing-superalignment/} {Introducing superalignment}.

\bibitem[{Ouyang et~al.(2022)Ouyang, Wu, Jiang, Almeida, Wainwright, Mishkin, Zhang, Agarwal, Slama, Ray et~al.}]{ouyang2022training}
Long Ouyang, Jeffrey Wu, Xu~Jiang, Diogo Almeida, Carroll Wainwright, Pamela Mishkin, Chong Zhang, Sandhini Agarwal, Katarina Slama, Alex Ray, et~al. 2022.
\newblock Training language models to follow instructions with human feedback.
\newblock \emph{Advances in neural information processing systems}, 35:27730--27744.

\bibitem[{Pinheiro~Cinelli et~al.(2021)Pinheiro~Cinelli, Ara{\'u}jo~Marins, Barros~da Silva, and Lima~Netto}]{pinheiro2021variational}
Lucas Pinheiro~Cinelli, Matheus Ara{\'u}jo~Marins, Eduardo~Ant{\'u}nio Barros~da Silva, and S{\'e}rgio Lima~Netto. 2021.
\newblock Variational autoencoder.
\newblock In \emph{Variational Methods for Machine Learning with Applications to Deep Networks}, pages 111--149. Springer.

\bibitem[{{Qwen Team}(2024)}]{qwen}
{Qwen Team}. 2024.
\newblock \href {https://qwenlm.github.io/blog/qwen2.5/} {Qwen2.5: A party of foundation models}.

\bibitem[{Radford et~al.(2019)Radford, Wu, Child, Luan, Amodei, Sutskever et~al.}]{radford2019language}
Alec Radford, Jeffrey Wu, Rewon Child, David Luan, Dario Amodei, Ilya Sutskever, et~al. 2019.
\newblock Language models are unsupervised multitask learners.
\newblock \emph{OpenAI blog}, 1(8):9.

\bibitem[{Rafailov et~al.(2024)Rafailov, Sharma, Mitchell, Manning, Ermon, and Finn}]{rafailov2024direct}
Rafael Rafailov, Archit Sharma, Eric Mitchell, Christopher~D Manning, Stefano Ermon, and Chelsea Finn. 2024.
\newblock Direct preference optimization: Your language model is secretly a reward model.
\newblock \emph{Advances in Neural Information Processing Systems}, 36.

\bibitem[{Sang et~al.(2024)Sang, Wang, Zhang, Zhu, Kong, Ye, Wei, and Xiao}]{sang2024improving}
Jitao Sang, Yuhang Wang, Jing Zhang, Yanxu Zhu, Chao Kong, Junhong Ye, Shuyu Wei, and Jinlin Xiao. 2024.
\newblock Improving weak-to-strong generalization with scalable oversight and ensemble learning.
\newblock \emph{arXiv preprint arXiv:2402.00667}.

\bibitem[{Shin et~al.(2025)Shin, Cooper, and Sala}]{shin2024weak}
Changho Shin, John Cooper, and Frederic Sala. 2025.
\newblock Weak-to-strong generalization through the data-centric lens.
\newblock In \emph{The Thirteenth International Conference on Learning Representations}.

\bibitem[{Somerstep et~al.(2024)Somerstep, Polo, Banerjee, Ritov, Yurochkin, and Sun}]{somerstep2024statistical}
Seamus Somerstep, Felipe~Maia Polo, Moulinath Banerjee, Ya'acov Ritov, Mikhail Yurochkin, and Yuekai Sun. 2024.
\newblock A statistical framework for weak-to-strong generalization.
\newblock \emph{arXiv preprint arXiv:2405.16236}.

\bibitem[{Somerstep et~al.(2025)Somerstep, Polo, Banerjee, Ritov, Yurochkin, and Sun}]{somerstep2025transfer}
Seamus Somerstep, Felipe~Maia Polo, Moulinath Banerjee, Yaacov Ritov, Mikhail Yurochkin, and Yuekai Sun. 2025.
\newblock A transfer learning framework for weak to strong generalization.
\newblock In \emph{The Thirteenth International Conference on Learning Representations}.

\bibitem[{Sun and van~der Schaar(2024)}]{sun2024inverse}
Hao Sun and Mihaela van~der Schaar. 2024.
\newblock Inverse-rlignment: Inverse reinforcement learning from demonstrations for llm alignment.
\newblock \emph{arXiv preprint arXiv:2405.15624}.

\bibitem[{Touvron et~al.(2023)Touvron, Martin, Stone, Albert, Almahairi, Babaei, Bashlykov, Batra, Bhargava, Bhosale et~al.}]{touvron2023llama}
Hugo Touvron, Louis Martin, Kevin Stone, Peter Albert, Amjad Almahairi, Yasmine Babaei, Nikolay Bashlykov, Soumya Batra, Prajjwal Bhargava, Shruti Bhosale, et~al. 2023.
\newblock Llama 2: Open foundation and fine-tuned chat models.
\newblock \emph{arXiv preprint arXiv:2307.09288}.

\bibitem[{Wang et~al.(2024)Wang, Jiang, Yang, Liu, and Chen}]{wang2023beyond}
Chaoqi Wang, Yibo Jiang, Chenghao Yang, Han Liu, and Yuxin Chen. 2024.
\newblock Beyond reverse kl: Generalizing direct preference optimization with diverse divergence constraints.
\newblock In \emph{International Conference on Learning Representations}.

\bibitem[{Wu and Sahai(2025)}]{wu2024provable}
David~X Wu and Anant Sahai. 2025.
\newblock Provable weak-to-strong generalization via benign overfitting.
\newblock In \emph{The Thirteenth International Conference on Learning Representations}.

\bibitem[{Wu et~al.(2025)Wu, Tao, Wang, Yang, Zhao, and Wong}]{wu2024rethinking}
Taiqiang Wu, Chaofan Tao, Jiahao Wang, Runming Yang, Zhe Zhao, and Ngai Wong. 2025.
\newblock Rethinking kullback-leibler divergence in knowledge distillation for large language models.
\newblock In \emph{Proceedings of the 31st International Conference on Computational Linguistics}, pages 5737--5755.

\bibitem[{Xue et~al.(2025)Xue, Li, and Mirzasoleiman}]{xue2025representations}
Yihao Xue, Jiping Li, and Baharan Mirzasoleiman. 2025.
\newblock Representations shape weak-to-strong generalization: Theoretical insights and empirical predictions.
\newblock \emph{arXiv preprint arXiv:2502.00620}.

\bibitem[{Yang et~al.(2025)Yang, Lin, Zhou, and Wen}]{yang-etal-2025-distilling}
Wenkai Yang, Yankai Lin, Jie Zhou, and Ji-Rong Wen. 2025.
\newblock Distilling rule-based knowledge into large language models.
\newblock In \emph{Proceedings of the 31st International Conference on Computational Linguistics}, pages 913--932.

\bibitem[{Yang et~al.(2024)Yang, Shen, Shen, Yao, Liu, Gong, Lin, and Wen}]{yang2024super}
Wenkai Yang, Shiqi Shen, Guangyao Shen, Wei Yao, Yong Liu, Zhi Gong, Yankai Lin, and Ji-Rong Wen. 2024.
\newblock Super (ficial)-alignment: Strong models may deceive weak models in weak-to-strong generalization.
\newblock \emph{arXiv preprint arXiv:2406.11431}.

\bibitem[{Yao et~al.(2025)Yao, Yang, Ziqiao, Lin, and Liu}]{yao2025understanding}
Wei Yao, Wenkai Yang, Wang Ziqiao, Yankai Lin, and Yong Liu. 2025.
\newblock Understanding the capabilities and limitations of weak-to-strong generalization.
\newblock \emph{arXiv preprint arXiv:2502.01458}.

\bibitem[{Ye et~al.(2024)Ye, Xiao, and Hui}]{ye2024weak}
Ruimeng Ye, Yang Xiao, and Bo~Hui. 2024.
\newblock Weak-to-strong generalization beyond accuracy: a pilot study in safety, toxicity, and legal reasoning.
\newblock \emph{arXiv preprint arXiv:2410.12621}.

\bibitem[{Ye et~al.(2025)Ye, Laidlaw, and Steinhardt}]{ye2025iterative}
Yaowen Ye, Cassidy Laidlaw, and Jacob Steinhardt. 2025.
\newblock Iterative label refinement matters more than preference optimization under weak supervision.
\newblock In \emph{The Thirteenth International Conference on Learning Representations}.

\bibitem[{Zhou et~al.(2025)Zhou, Shen, and Cheng}]{zhou2025weak}
Yucheng Zhou, Jianbing Shen, and Yu~Cheng. 2025.
\newblock Weak to strong generalization for large language models with multi-capabilities.
\newblock In \emph{The Thirteenth International Conference on Learning Representations}.

\bibitem[{Zhu et~al.(2025)Zhu, He, Wang, Liu, and Wang}]{zhu2024weak}
Wenhong Zhu, Zhiwei He, Xiaofeng Wang, Pengfei Liu, and Rui Wang. 2025.
\newblock Weak-to-strong preference optimization: Stealing reward from weak aligned model.
\newblock In \emph{The Thirteenth International Conference on Learning Representations}.

\end{thebibliography}

\newpage
\onecolumn
\appendix

\clearpage
\tableofcontents
\newpage

{\LARGE \centering \textbf{Appendix} \par}
\section{Discussion of Concurrent Work} \label{discussion:concurrent}

Recently, a concurrent work~\citep{mulgund2025relating} has also independently addressed a problem similar to ours. 
The primary similarity is found in~\cref{sec:upper}.
In particular, our~\cref{thm:realizable} in~\cref{sec:upper} is similar to their Theorem 4.1 \& Corollary 4.2.
The proof of their Theorem 4.1 \& Corollary 4.2 and our~\cref{thm:realizable} share the convexity assumption and the mathematical formulation of Bregman divergence. However, the proof techniques differ significantly. While we explicitly construct the sum of first-order and second-order terms through derivation and calculation, they apply mathematical tools such as generalized Pythagorean inequality, convex analysis, and the sequential consistency property to derive their results.
Building on the derived results, while we focus on overcoming the realizability assumption and deriving sample complexity bounds in our~\cref{thm:non-realizable-finite-samples}, they aim to relax the convexity assumption by projecting the weak model onto convex combinations of functions based on the strong model representation in their Theorem 4.3.
Both our work and~\citet{mulgund2025relating} contribute to the theoretical understanding of W2SG in classification, with a particular focus on reverse KL/CE losses.

\section{Main Proof} 

\subsection{Proof of~\cref{lemma:upper_lower_inf}} \label{proof_lemma_inf}

We first define the corresponding probability distributions for prediction of $F_{sw}$ and $F_w$.
$\forall x \in \cX$, we know that $\sum_{j=1}^k [F_w(x)]_j = 1$. Therefore, given the class space $C_k = \{ 1, \cdots, k \}$, we define a probability distribution $\cP_{w}(x)$ with the probability density function $p_w$, where $j \in C_k$ and 
\begin{align} \label{def_new_distribution}
    p_w(j)=[F_w(x)]_j.
\end{align}
Using this method, we also define the probability distribution $\cP_{sw}(x)$ for $F_{w}(x)$.

\begin{lemma}[\citet{yao2025understanding}] \label{infor_lemma}
Given the probability distributions $\cP_{w}(x)$ and $\cP_{sw}(x)$ above.
For any $x \in \cX$, $j \in C_k$, $g: C_k \to \R$ and assume that $g$ is $\sigma$-subgaussian
\footnote{A random variable $X \in \R$ is $\sigma$-subgaussian if for any $\rho$, $\log \mathbb{E} \exp (\rho(X-\mathbb{E} X)) \leq \rho^2 \sigma^2 / 2$.}.
Let $f=t \cdot g$ for any $t \in \R$, then 
$$\mathrm{D}_{\mathrm{KL}}\left( F_w(x) \| F_{sw}(x) \right) \ge \sup _t t\left(\mathbb{E}_{j^{\prime} \sim \cP_w(x)}\left[g\left(j^{\prime}\right)\right]-\mathbb{E}_{j \sim \cP_{sw}(x)}[g(j)]\right)-t^2 \sigma^2 / 2.$$
\end{lemma}

Now we start the proof. 
The derivations are mainly adapted from~\citet{yao2025understanding} by swapping $F_w$ and $F_{sw}$ and using the connection between cross-entropy and KL divergence. 
To make this paper self-contained, we incorporate the proof below.

\begin{proof}

By taking expectations of $x$ on both sides of the inequality in~\cref{infor_lemma}, we obtain
\begin{multline*}
    \kl(F_w, F_{sw}) = \bE_x \mathrm{D}_{\mathrm{KL}}\left( F_w(x) \| F_{sw}(x) \right) \\ \ge \sup _t \underbrace{t\left(\bE_x\mathbb{E}_{j^{\prime} \sim \cP_w(x)}\left[g\left(j^{\prime}\right)\right]-\bE_x\mathbb{E}_{j \sim \cP_{sw}(x)}[g(j)]\right)-t^2 \sigma^2 / 2}_{\phi(t)}.
\end{multline*}

Note that $\phi(t)$ is a quadratic function of $t$.
Therefore, by AM–GM inequality, we find the maximum of this quadratic function:
\begin{align*}
    \phi(t) \le \frac{1}{2\sigma^2}\left(\bE_x\mathbb{E}_{j^{\prime} \sim \cP_w(x)}\left[g\left(j^{\prime}\right)\right]-\bE_x\mathbb{E}_{j \sim \cP_{sw}(x)}[g(j)]\right)^2 = \sup _t \phi(t) \le \kl(F_w, F_{sw}).
\end{align*}

Subsequently, there holds
\begin{align} \label{ineq:lower_upper_kl_loss}
\left|\bE_x\mathbb{E}_{j^{\prime} \sim \cP_w(x)}\left[g\left(j^{\prime}\right)\right]-\bE_x\mathbb{E}_{j \sim \cP_{sw}(x)}[g(j)]\right| \le \sqrt{2\sigma^2 \kl(F_w, F_{sw})}.
\end{align}

Likewise, according to~\cref{infor_lemma}, we have
\begin{align} \label{proof:variant-1}
    \mathrm{D}_{\mathrm{KL}}\left( F_{sw}(x) \| F_w(x) \right) \ge \sup _t t\left(\mathbb{E}_{j \sim \cP_{sw}(x)}\left[g\left(j^{\prime}\right)\right]-\mathbb{E}_{j^{\prime} \sim \cP_w(x)}[g(j)]\right)-t^2 \sigma^2 / 2.
\end{align}

We apply the same proof technique to~\eqref{proof:variant-1} and obtain:
\begin{align} \label{proof:variant-2}
    \left|\bE_x\mathbb{E}_{j^{\prime} \sim \cP_w(x)}\left[g\left(j^{\prime}\right)\right]-\bE_x\mathbb{E}_{j \sim \cP_{sw}(x)}[g(j)]\right| \le \sqrt{2\sigma^2 \kl(F_{sw}, F_w)}.
\end{align}

Now we construct $g$ to associate the above results with $\dist(F^\star, F_{sw})$ and $\dist\left( F^\star, F_w  \right)$.
Specifically, given a probability distribution $\cP_g$ with the density function $p_g$, we define function $g: C_k \to (0,1]$ associated with $\cP_g$: 
$$g(j) \triangleq \frac{[F^\star(x)]_j}{p_g(j)} \log \frac{[F^\star(x)]_j}{p_g(j)}, \quad \text{for} \ j \in C_k.$$

We have
\begin{align*}
    \bE_x\mathbb{E}_{j \sim \cP_g} \left[g(j)\right] & = \bE_x \bE_{j \sim \cP_g} \left[\frac{[F^\star(x)]_j}{p_g(j)} \log \frac{[F^\star(x)]_j}{p_g(j)} \right] \\
    & = \bE_x \left[\sum_{j \in C_k} p_g(j) \cdot \frac{[F^\star(x)]_j}{p_g(j)} \cdot \log \frac{[F^\star(x)]_j}{p_g(j)} \right] \\
    & = \bE_x \left[\sum_{j \in C_k} [F^\star(x)]_j \cdot \log \frac{[F^\star(x)]_j}{p_g(j)} \right]
\end{align*}

Recall the definition of $\cP_{sw}$ and $\cP_w$ in~\eqref{def_new_distribution}, we replace $\cP_g$ with $\cP_{sw}$ and $\cP_w$ in the above equation:
\begin{align*}
    & \bE_x\mathbb{E}_{j^{\prime} \sim \cP_{sw}}\left[g\left(j^{\prime}\right)\right] = \bE_x \left[ \sum_{j=1} [F^\star(x)]_j \log \frac{[F^\star(x)]_j}{[F_{sw}(x)]_j} \right] = \kl(F^\star, F_{sw}), \\
    & \bE_x\mathbb{E}_{j \sim \cP_w}[g(j)] = \bE_x \left[ \sum_{j=1} [F^\star(x)]_j \log \frac{[F^\star(x)]_j}{[F_{w}(x)]_j} \right] = \kl(F^\star, F_{w}).
\end{align*}

Substitute the above into~\eqref{ineq:lower_upper_kl_loss}:
\begin{align} \label{ineq:temp-1}
    \left| \dist(F^\star, F_{sw})-\dist(F^\star, F_{w}) \right| \le \sqrt{2\sigma^2 \kl(F_w, F_{sw})},
\end{align}
The above inequality is because whether $\dist$ is $\kl$ or $\cross$, we have 
$$\dist(F^\star, F_{sw})-\dist(F^\star, F_{w}) = \kl(F^\star, F_{sw})-\kl(F^\star, F_{w}).$$

Likewise, we apply the same proof technique to~\eqref{proof:variant-2} and obtain:
\begin{align} \label{proof:variant-3}
    \left| \dist(F^\star, F_{sw})-\dist(F^\star, F_{w}) \right| \le \sqrt{2\sigma^2 \kl(F_{sw}, F_w)}.
\end{align}

Finally, we obtain the subgaussian factor $\sigma$ of function $g$ by using the fact that $g$ is bounded.
Recall that the output domain $\cY \subseteq \R^k$, where $\forall y = (y_1, \cdots, y_k)^T \in \cY$, there holds $\sum_{i=1}^k y_i=1$ and $0 < y_i \le 1$.
In other words, $\exists \gamma>0$ such that $0 < \gamma \le y_i \le 1$.
It means that $g(j) \in [-\frac{1}{\gamma} \log \frac{1}{\gamma}, \frac{1}{\gamma} \log \frac{1}{\gamma}]$.
According to Hoeffding’s lemma, $\forall \lambda \in \R$, we have
$$\mathbb{E}\left[e^{\lambda(g(j)-\mathbb{E}[g(j)])}\right] \leq \exp \left(\frac{\lambda^2 \left(\frac{1}{\gamma} \log \frac{1}{\gamma} \right)^2}{2}\right).$$
In other words, $g(j)$ is $\sigma$-subgaussian, where $\sigma=\frac{1}{\gamma} \log \frac{1}{\gamma}$.
Substitute it into~\eqref{ineq:temp-1} and~\eqref{proof:variant-3}, we prove the final results:
\begin{align*}
    & \left| \dist(F^\star, F_{sw}) - \dist\left( F^\star, F_w  \right) \right| \le C_1 \sqrt{\kl(F_w, F_{sw})}, \\
    & \left| \dist(F^\star, F_{sw}) - \dist\left( F^\star, F_w  \right) \right| \le C_1 \sqrt{\kl(F_{sw}, F_w)},
\end{align*}
where the constant $C_1 = \frac{\sqrt{2}}{\gamma} \log \frac{1}{\gamma}$.

\end{proof}

\subsection{Proof of~\cref{theorem:residue}} \label{constant:theorem}

Total variation distance is introduced for our proof.
\begin{definition}[Total Variation Distance] \label{def:tv_distance}
Given two probability distributions $P$ and $Q$, the Total Variation (TV) distance between $P$ and $Q$ is
$$\tv(P \| Q)= \frac{1}{2} \int_{x \in \mathcal{X}} \left| P(x)-Q(x) \right| d x.$$
\end{definition}
Note that $\tv(P \| Q)\in[0,1]$. Also, $\tv(P \| Q)=0$ if and only if $P$ and $Q$ coincides, and $\tv(P \| Q)=1$ if and only if $P$ and $Q$ are disjoint.

\begin{proof}
We have
\begin{align}
    \dist(F^\star, F_w) &= \bE_x \left[ \sum_{i=1}^k [F^\star(x)]_i \log \frac{[F^\star(x)]_i}{[F_w(x)]_i} \right] \nonumber \\
    &= \bE_x \left[ \sum_{i=1}^k [F^\star(x)]_i \log \left( \frac{[F^\star(x)]_i}{[F_{sw}(x)]_i} \cdot \frac{[F_{sw}(x)]_i}{[F_w(x)]_i} \right) \right] \nonumber \\
    &= \bE_x \left[ \sum_{i=1}^k [F^\star(x)]_i \log  \frac{[F^\star(x)]_i}{[F_{sw}(x)]_i} \right] + \bE_x \left[ \sum_{i=1}^k [F^\star(x)]_i \log  \frac{[F_{sw}(x)]_i}{[F_w(x)]_i} \right] \nonumber \\
    & = \dist(F^\star, F_{sw}) + \left \langle F^\star, \log{\frac{F_{sw}}{F_w}} \right \rangle_E.
\end{align}
Rearranging terms and we know that:
\begin{align} \label{eqn:temp_b3}
    \dist(F^\star, F_{sw}) = \dist(F^\star, F_w) - \left \langle F^\star, \log{\frac{F_{sw}}{F_w}} \right \rangle_E.
\end{align}

Recall that the output domain $\cY \subseteq \R^k$, where $\forall y = (y_1, \cdots, y_k)^T \in \cY$, there holds $\sum_{i=1}^k y_i=1$ and $0 < y_i \le 1$.
In other words, $\exists \gamma>0$ such that $0 < \gamma \le y_i \le 1$.
Firstly, we know that $F^\star(x)$ is element-wise non-negative.
Denote $\vec{1}=(1,1, \cdots, 1)^T$. We know that there is a positive constant $\frac{1}{\gamma} \ge \left(\min_i [F_w(x)]_i \right)^{-1}$.
We use element-wise addition, subtraction, multiplication, division and absolute value in the proof.
Note that 
\begin{align*}
    \left \langle F^\star, \log{\frac{F_{sw}}{F_w}} \right \rangle_E & \le \left \langle F^\star, \frac{F_{sw}}{F_w} -\vec{1} \right \rangle_E \tag{$\log x \le x-1$} \\
    & \le \left \langle F^\star, \frac{1}{\gamma} \cdot F_w \left\vert \frac{F_{sw}}{F_w} -\vec{1} \right\vert \right \rangle_E \tag{$\frac{1}{\gamma} \cdot F_w \ge \vec{1}$ (element-wise)} \\
    & = \frac{1}{\gamma} \cdot \left \langle F^\star, \left\vert F_{sw}-F_w \right\vert  \right \rangle_E,
\end{align*}

and
\begin{align*}
    \left \langle F^\star, \left\vert F_{sw}-F_w \right\vert \right \rangle_E & = \expect_x \left[ \left(F^\star(x)\right)^T \left( \left|F_{sw}(x)-F_w(x)\right| \right) \right] \\
    & \le \expect_x \left[ \left\| F^\star(x) \right\|_\infty \cdot \left\| F_{sw}(x)-F_w(x) \right\|_1 \right] \tag{Holder’s inequality for vector-valued functions} \\ 
    & \le \expect_x \left[ \left\| F_{sw}(x)-F_w(x) \right\|_1 \right] \tag{$[F^\star(x)]_i \le 1$} \\ 
    & = 2 \expect_x \tv \left( F_w(x), F_{sw}(x) \right) \tag{Definition of TV distance} \\
    & \le 2 \sqrt{\expect_x \tv^2 \left( F_w(x), F_{sw}(x) \right)} \tag{Jensen’s inequality} \\ 
    & \le 2 \sqrt{\frac{1}{2}\expect_x \mathrm{D}_{\mathrm{KL}} \left( F_w(x), F_{sw}(x) \right)} \tag{Pinsker’s inequality} \\ 
    & = \sqrt{2\kl(F_w, F_{sw})}. \tag{Definition of $\kl(\cdot, \cdot)$}
\end{align*}

Therefore, 
$$\left \langle F^\star, \log{\frac{F_{sw}}{F_w}} \right \rangle_E \le \frac{\sqrt{2}}{\gamma} \cdot \sqrt{\kl(F_w, F_{sw})}.$$
Since the TV distance is symmetric, we also have
$$\left \langle F^\star, \log{\frac{F_{sw}}{F_w}} \right \rangle_E \le \frac{\sqrt{2}}{\gamma} \cdot \sqrt{\kl(F_{sw}, F_w)}.$$
Substitute them into~\cref{eqn:temp_b3} and we can prove that:
\begin{align*}
    & \dist(F^\star, F_{sw}) \ge \dist\left( F^\star, F_w \right) - \underbrace{\frac{\sqrt{2}}{\gamma}}_{C_2} \sqrt{\kl(F_w, F_{sw})}, \\
    & \dist(F^\star, F_{sw}) \ge \dist\left( F^\star, F_w \right) - \underbrace{\frac{\sqrt{2}}{\gamma}}_{C_2} \sqrt{\kl(F_{sw}, F_w)}.
\end{align*}

The above inequalities also applies to $\dist(\cdot, \cdot)=\cross(\cdot, \cdot)$ because whether $\dist$ is $\kl$ or $\cross$, we have 
$$\dist(F^\star, F_{sw})-\kl(F^\star, F_{sw}) = \dist(F^\star, F_{w})-\kl(F^\star, F_{w}).$$

\end{proof}

\paragraph{Discussion of the constant.}
Recall that $C_1 = \frac{\sqrt{2}}{\gamma} \log \frac{1}{\gamma}$ and $C_2 = \frac{\sqrt{2}}{\gamma}$. 
In other words, $\gamma < \frac{1}{e}$ leads to $C_2 \le C_1$.
While $\gamma$ is the minimal value of the output, it is generally very small ($\gamma=10^{-3} \ll \frac{1}{e}$ in our experiments), i.e., $C_2 \le C_1$.
Therefore, the lower bound in~\cref{theorem:residue} is tighter than that in~\cref{lemma:upper_lower_inf}.

\paragraph{Further Discussion.}
We show that adding an additional assumption leads to $\dist(F^\star, F_{sw}) \ge \dist\left( F^\star, F_w \right) - \dist(F_w, F_{sw})$.
Particularly, if $\dist(F_w, F_{sw})$ can be improved to some extent, the constant $C$ and square root in~\cref{theorem:residue} can be eliminated, contributing to a more elegant version:
\begin{corollary}\label{corollary:residue_large_dp}
    Let $\dist$ to be $\kl$ or $\cross$.
    Let $R \ge 0$ and consider the same constant $C$ in~\cref{theorem:residue}.
    If $\dist(F_w, F_{sw}) \ge \sqrt{2}C$ is satisfied, then: 
    \begin{align*}
        \dist(F^\star, F_{sw}) \ge \dist\left( F^\star, F_w \right) - \kl(F_w, F_{sw}).
    \end{align*}
\end{corollary}

\cref{corollary:residue_large_dp} removes the constant coefficient and square root from~\cref{theorem:residue}.
Notably, if $R \ge 0$, the results above reinforce that the key bottleneck for performance improvement over $F_w$ arises from the optimization objective's inherent nature~\citep{yao2025understanding}: 
If $\dist(F_w, F_{sw})$ can be large, the performance improvement cannot exceed $\dist(F_w, F_{sw})$, which is exactly the minimum of~\cref{eqn:fsw-population-minimizer}.

\begin{proof}
We adopt an alternative proof technique in the proof of~\cref{theorem:residue}:
\begin{align*}
    \left| \left \langle F^\star, \left\vert F_{sw}-F_w \right\vert \right \rangle_E \right| & \le 2 \expect_x \tv \left( F_w(x), F_{sw}(x) \right) \tag{The derivations in~\cref{constant:theorem}} \\
    & \le 2 \expect_x \sqrt{1- \exp{\left[-\mathrm{D}_{\mathrm{KL}} \left( F_w(x), F_{sw}(x) \right)\right]}} \tag{Bretagnolle–Huber inequality} \\ 
    & \le 2 \sqrt{1- \exp{\left[-\expect_x \mathrm{D}_{\mathrm{KL}} \left( F_w(x), F_{sw}(x) \right)\right]}} \tag{Jensen’s inequality} \\ 
    & = 2 \sqrt{1- \exp{\left(-\kl(F_w, F_{sw})\right)}}. \tag{Definition of $\kl$}
\end{align*}

Let $u(t)=e^{-t} + \frac{\gamma^2}{4}t^2 - 1, t \ge 0$.
Taking the first-order and second-order derivative:
$u'(t)=-e^{-t} + \frac{\gamma^2}{2}t$, and $u''(t)=e^{-t} + \frac{\gamma^2}{2} > 0$.
While $u'(0)=-1<0$, $u'(\frac{2}{\gamma^2}) >0$, we know that there only exists a $t_0 \in (0,\frac{2}{\gamma^2})$ such that $u'(t_0)=0$.
And $u(t)$ decreases at $[0,t_0]$, increases at $(t_0, +\infty)$ and $u(0)=0$.
Denote $u(t^\star)=0$.
Notice that $u(\frac{2}{\gamma})=e^{-\frac{2}{\gamma}} > 0$, which means that $t^\star < \frac{2}{\gamma}$.
In other words, $t > \frac{2}{\gamma}$ leads to $u(t)>0$, i.e., $\sqrt{1-e^{-t}} \le \frac{\gamma}{2}t$.

Using the above results, if $\left \langle F^\star, \log{\frac{F_{sw}}{F_w}} \right \rangle_E \ge 0$ and $\kl(F_w, F_{sw}) \ge \frac{2}{\gamma}$, then 
\begin{align*}
    \left| \left \langle F^\star, \log{\frac{F_{sw}}{F_w}} \right \rangle_E \right| & \le \left| \frac{1}{\gamma} \cdot \left \langle F^\star, \left\vert F_{sw}-F_w \right\vert  \right \rangle_E \right| \tag{The derivations in~\cref{constant:theorem}} \\ 
    & \le \frac{2}{\gamma} \sqrt{1- \exp{\left(-\kl(F_w, F_{sw})\right)}} \\
    & \le \frac{2}{\gamma} \cdot \frac{\gamma}{2} \kl(F_w, F_{sw}) \\
    & = \kl(F_w, F_{sw}).
\end{align*}

The proof is complete.

\end{proof}

\subsection{Proof of~\cref{prop:general_equation}} \label{proof:general_equation}

\begin{proof}
We have
\begin{align*}
    \dist(F^\star, F_w) &= \bE_x \left[ \sum_{i=1}^k [F^\star(x)]_i \log \frac{[F^\star(x)]_i}{[F_w(x)]_i} \right] \\
    &= \bE_x \left[ \sum_{i=1}^k [F^\star(x)]_i \log \left( \frac{[F^\star(x)]_i}{[F_{sw}(x)]_i} \cdot \frac{[F_{sw}(x)]_i}{[F_w(x)]_i} \right) \right] \\
    &= \bE_x \left[ \sum_{i=1}^k [F^\star(x)]_i \log  \frac{[F^\star(x)]_i}{[F_{sw}(x)]_i} \right] + \bE_x \left[ \sum_{i=1}^k [F^\star(x)]_i \log  \frac{[F_{sw}(x)]_i}{[F_w(x)]_i} \right] \\
    &=\dist(F^\star, F_{sw}) + \left \langle F^\star, \log{\frac{F_{sw}}{F_w}} \right \rangle_E.
\end{align*}
Rearranging terms and we can prove the result.

The above also applies to $\dist(\cdot, \cdot)=\cross(\cdot, \cdot)$ because whether $\dist$ is $\kl$ or $\cross$, we have 
$$\dist(F^\star, F_{sw})-\kl(F^\star, F_{sw}) = \dist(F^\star, F_{w})-\kl(F^\star, F_{w}).$$
\end{proof}

\paragraph{Insights for reverse KL loss.}
Using similar decomposition technique, we obtain
\begin{align*}
    \dist(F_w, F^\star) = \dist(F_{sw}, F^\star) + \underbrace{\left \langle F_w-F_{sw}, \log{\frac{F_w}{F^\star}} \right \rangle_E - \dist(F_{sw}, F_w)}_{R_1}.
\end{align*}

Therefore, $\dist(F_{sw}, F^\star) \le \dist\left( F_w, F^\star \right)$ satisfies \textit{if and only if} $R_1 \ge 0$.
While the teacher-student disagreement is minimized in W2SG, we expect a small value of $\dist(F_{sw}, F_w)$.
Therefore, we want to obtain a large $\left \langle F_w-F_{sw}, \log{\frac{F_w}{F^\star}} \right \rangle_E$.
Intuitively, for any $x \in \cX$ and $i \in \{ 1, \cdots, k  \}$, we expect the model predictions to satisfy either of the two inequalities:
\begin{align}
    & [F_w(x)]_i \ge \max([F_{sw}(x)]_i ,[F^\star(x)]_i), \label{prop_2_ineq_1-r} \\ 
    & [F_w(x)]_i \le \min([F_{sw}(x)]_i ,[F^\star(x)]_i). \label{prop_2_ineq_2-r}
\end{align}
In other words, the predicted probabilities of $F_{sw}$ reflect the true outcome better than $F_w$.
The confidence level of $F_{sw}$ should be better aligned with $F^\star$ than that of $F_w$.

\paragraph{Insights for squared loss.} 
\citet{charikar2024quantifying} consider the squared loss:
$$\dist(f, g)=\mathbb{E}_{x \sim \mathcal{P}}(f(x)-g(x))^2.$$

In this setting, $\dist(f, g)=\dist(g, f)$ and we have
\begin{align*}
\dist\left(F_w, F^\star\right) & = \mathbb{E}_{x \sim \mathcal{P}}\left(F^\star(x)-F_w(x)\right)^2 \\
& = \mathbb{E}_{x \sim \mathcal{P}}\left(F^\star(x)-F_{sw}(x)+F_{sw}(x)-F_w(x)\right)^2 \\
& = \mathbb{E}_{x \sim \mathcal{P}}\left(F^\star(x)-F_{sw}(x)\right)^2+\mathbb{E}_{x \sim \mathcal{P}}\left(F_{sw}(x)-F_w(x)\right)^2 \\ & \hspace{2cm} +2 \cdot \mathbb{E}_{x \sim \mathcal{P}}\left[\left(F^\star(x)-F_{sw}(x)\right)\left(F_{sw}(x)-F_w(x)\right)\right] \\
& = \dist\left(F_{s w}, F^\star\right)+\dist\left(F_{s w}, F_w\right) +2 \cdot \mathbb{E}_{x \sim \mathcal{P}}\left[\left(F^\star(x)-F_{sw}(x)\right)\left(F_{sw}(x)-F_w(x)\right)\right].
\end{align*}

If we define
$$\left \langle f,g \right \rangle_S = 2 \cdot \mathbb{E}_{x \sim \mathcal{P}} \left[ f(x) \cdot g(x) \right],$$
then we have
\begin{align*}
    \dist\left(F_w, F^\star\right) = \dist\left(F_{s w}, F^\star\right)+\dist\left(F_{s w}, F_w\right) + \left \langle F^\star - F_{sw}, F_{sw}-F_w \right \rangle_S.
\end{align*}

Rearranging terms and we have
\begin{align} \label{suff_necc_condition_squared}
    \dist\left(F_{s w}, F^\star\right) = \dist\left(F_w, F^\star\right)-\dist\left(F_{s w}, F_w\right) - \left \langle F^\star - F_{sw}, F_{sw}-F_w \right \rangle_S.
\end{align}

Therefore, $\left \langle F^\star - F_{sw}, F_{sw}-F_w \right \rangle_S > 0$ is the sufficient and necessary condition for the inequality $$\dist(F_{sw}, F^\star) \le \dist(F_w, F^\star) - \dist(F_w, F_{sw}),$$
when $\dist$ is the squared loss.
Therefore, we should make the confidence level of $F_{sw}$ better aligned with $F^\star$.
Despite the difficulty to attain this objective, \citet{charikar2024quantifying} demonstrate that, within an elegant proof framework using convexity assumption, this condition is guaranteed to hold.

\subsection{Proof of~\cref{thm:realizable}} \label{theorem1_kl_loss}

\paragraph{Proof sketch.}
We define $\kl(\cdot, \cdot)$ in a Bregman-divergence manner.
To derive the desired properties, we construct a convex combination of the form $F_{sw}(x)+t(F^\star(x)-F_{sw}(x))$, where $t \to 0^+$.
By analyzing this construction, we show that the sum of the first-order term $\cO(t)$ and the second-order term $\cO(t^2)$ is non-negative.
This implies that the first-order term itself must also be non-negative.
Leveraging this principle and the associated derivations, we establish the proof of our results.

Our proof technique is general and unifying, covering \textit{both squared loss and KL divergence loss}. 
While Theorem 1-2 from~\citet{charikar2024quantifying} focus exclusively on squared loss in regression, and Theorem 3-4 from~\citet{yao2025understanding} restrict the analysis to KL divergence-like loss in regression, 
our work extends the scope to classification problems, encompassing both squared loss and KL divergence loss in a single framework. 
This broader applicability highlights the versatility of our proof and its potential to bridge gaps between regression and classification settings.
We recently became aware of concurrent work by~\citet{mulgund2025relating}, which has independently developed a theoretical framework employing a similar proof technique. As discussed in~\cref{discussion:concurrent}, while there are some conceptual overlaps, the core proof methodologies in our work differ significantly.

We first restate a lemma for our proof.
Let the strong model learns from $\cF_s:\R^{d_s} \to \R$ (which is a convex set) of fine-tuning tasks. 
Recall that we denote the strong model representation map by $h_s:\R^d \to \R^{d_s}$. Let $V_s = \{f \circ h_s: f \in \cF_s\}$ be the set of all tasks in $\cF_s$ composed with the strong model representation. 
Then $V_s$ is also a convex set.
\begin{lemma}[\citet{charikar2024quantifying}]
    \label{claim:Vs-convex}
    $V_s$ is a convex set.
\end{lemma}

\begin{proof}
Fix $f, g \in \cF_s$, and consider $f \circ h_s, g \circ h_s \in V_s$. Fix any $\lambda \in [0,1]$. Since $\cF_s$ is the linear function class so that it is a convex set, there exists $p \in \cF_s$ such that for all $y \in \R^{d_s}$, $p(y) = \lambda f(y) + (1-\lambda)g(y)$. Now, fix any $x \in \R^d$. Then, we have that
\begin{align*}
    \lambda (f \circ h_s)(x) + (1-\lambda)(g \circ h_s)(x) &= \lambda f(h_s(x)) + (1-\lambda)g(h_s(x))
    = p(h_s(x)) = (p \circ h_s)(x),
\end{align*}
and hence $\lambda (f \circ h_s) + (1-\lambda)(g \circ h_s) = p \circ h_s \in V_s$, which means that $V_s$ is also a convex set.
\end{proof}

We then present our theoretical results. 

Motivated by the definition of Bregman divergence, we consider $\dist$ as:
\begin{align}
    \dist(F_1, F_2) = \expect_x \left[ \phi(F_1(x))-\phi(F_2(x)) - \left \langle \nabla \phi(F_2(x)), F_1(x)-F_2(x) \right \rangle \right],
\end{align}
where $F_1, F_2 \in \cF$, and $\phi: \R^k \to \R$ is a strictly convex and differentiable function.
Note that squared loss and KL divergence loss are special cases of the definition of $\dist$ above:
\begin{align*}
    & \textbf{Squared loss:} \quad \dist(F_1, F_2)=\expect_x\|F_1(x)-F_2(x) \|_2^2, \quad \phi(x)=x^Tx, \\
    & \textbf{KL divergence loss:} \quad \dist(F_1, F_2)= \expect_x \sum_{i=1}^k [F_1(x)]_i \log \frac{[F_1(x)]_i}{[F_2(x)]_i}, \quad \phi(x)=\sum_{i=1}^k x_i \log x_i.
\end{align*}

Now we start our proof of~\cref{thm:realizable}.

\begin{proof}

We observe that
\begin{align*}
    & \dist(g, F_w) = \expect_x \left[ \phi(g(x)) - \phi(F_w(x)) - \left \langle \nabla \phi(F_w(x)), g(x)-F_w(x) \right \rangle \right], \\
    & \dist(g, F_{sw})  = \expect_x \left[ \phi(g(x)) - \phi(F_{sw}(x)) - \left \langle \nabla \phi(F_{sw}(x)), g(x)-F_{sw}(x) \right \rangle \right], \\
    & \dist(F_{sw}, F_w) = \expect_x \left[\phi(F_{sw}(x)) - \phi(F_w(x)) - \left \langle \nabla \phi(F_w(x)), F_{sw}(x)-F_w(x) \right \rangle \right],
\end{align*}
which means that
\begin{align} \label{dist_expansion_bregman-rev}
    \dist(g, F_w) = \dist(g, F_{sw}) + \dist(F_{sw}, F_w) + \underbrace{\expect_x \left \langle g(x)-F_{sw}(x), \nabla \phi(F_{sw}(x)) - \nabla \phi(F_w(x)) \right \rangle}_{R_1}.
\end{align}

Now our goal is to prove that $R_1 \ge 0$.
We use reverse KL as the loss function in W2SG: $f_{sw} = \argmin_{f \in \cF}\; \dist(f \circ h_s, F_w)$.
In other words, $F_{sw}$ is the \textit{projection} of $F_w$ onto the convex set $V_s$, i.e., $\dist(g, F_w) \ge \dist(F_{sw}, F_w)$.
Substitute it into~\cref{dist_expansion_bregman-rev} and we have
\begin{align} \label{bregman:disc-rev}
    R_1 + \dist(g, F_{sw}) \ge 0.
\end{align}

\paragraph{Case 1: squared loss.}
Let $g=F_{sw} + t(F^\star-F_{sw})$, $t \in (0,1)$, $t \to 0^+$.
Consider $\phi(x)=x^Tx$, so $\nabla \phi(x)=2x$.
There holds
\begin{align*}
    & R_1 = 2t \cdot \expect_x \left \langle F_{sw}(x)-F_w(x), F^\star(x)-F_{sw}(x) \right \rangle = \cO(t), \\
    & \dist(g, F_{sw}) = t^2 \expect_x \| F^\star(x) - F_{sw}(x) \|_2^2 = \cO(t^2).
\end{align*}
Recall~\cref{dist_expansion_bregman-rev} that $\cO(t) + \cO(t^2) \ge 0$, which means that $\cO(t) \ge 0$.
Therefore, there holds $R_1 \ge 0$, which means
$$\expect_x \left \langle F^\star(x)-F_{sw}(x), \nabla \phi(F_{sw}(x)) - \nabla \phi(F_w(x)) \right \rangle \ge 0.$$
Let $g=F^\star$ in~\cref{dist_expansion_bregman-rev} and we can prove the result $\dist(F^\star, F_{sw}) \le \dist(F^\star, F_w) - \dist(F_{sw}, F_w)$.
While our proof is different from~\citet{charikar2024quantifying}, we obtain the same conclusion for squared loss.

\paragraph{Case 2: reverse KL divergence.}
We consider $\dist(\cdot, \cdot)=\kl(\cdot, \cdot)$.
Let $g=F_{sw} + t(F^\star-F_{sw})$, $t \in (0,1)$, $t \to 0^+$.
Consider $\phi(x)=\sum_{i=1}^k x_i \log x_i$, so $\nabla \phi(x)= [\log x_1+1, \cdots, \log x_k+1]^T$.
Firstly,
\begin{align*}
    R_1 = t \cdot \expect_x (F^\star(x)-F_w(x))^T 
    \begin{bmatrix}
     \log \frac{[F_{sw}(x)]_1}{[F_w(x)]_1} \\
     \vdots \\
    \log \frac{[F_{sw}(x)]_k}{[F_w(x)]_k}
    \end{bmatrix} = \cO(t).
\end{align*}

Secondly,
\begin{align*}
    \dist(g, F_{sw}) & = \expect_x \sum_{i=1}^k [g(x)]_i \log \frac{[g(x)]_i}{[F_{sw}(x)]_i} 
    \\ & = \expect_x \sum_{i=1}^k [F_{sw}(x) + t(F^\star(x)-F_{sw}(x))]_i \log \left( 1+ t \cdot \frac{[F^\star(x)-F_{sw}(x)]_i}{[F_{sw}(x)]_i} \right)
    \\ & = \expect_x \sum_{i=1}^k [F_{sw}(x) + t(F^\star(x)-F_{sw}(x))]_i \left( t \cdot \frac{[F^\star(x)-F_{sw}(x)]_i}{[F_{sw}(x)]_i} + \cO(t^2) \right) \tag{Taylor expansion}
    \\ & = \expect_x \sum_{i=1}^k [F_{sw}(x)]_i \left( t \cdot \frac{[F^\star(x)-F_{sw}(x)]_i}{[F_{sw}(x)]_i} + \cO(t^2) \right) + \cO(t^2) 
    \\ & = t \cdot \expect_x \sum_{i=1}^k [F^\star(x)-F_{sw}(x)]_i + \cO(t^2) 
    \\ & = \cO(t^2),
\end{align*}
where the last equation is because $\expect_x \sum_{i=1}^k [F^\star(x)]_i = \expect_x \sum_{i=1}^k [F_{sw}(x)]_i = 1$.
Therefore,
$$\underbrace{R_1}_{\cO(t)} + \underbrace{\dist(g, F_{sw})}_{\cO(t^2)} \ge 0,$$
which means $R_1 \ge 0$, i.e.,
$$\expect_x \left \langle F^\star(x)-F_{sw}(x), \nabla \phi(F_{sw}(x)) - \nabla \phi(F_w(x)) \right \rangle \ge 0.$$
Let $g=F^\star$ in~\cref{dist_expansion_bregman-rev} and we can prove the result $\dist(F^\star, F_{sw}) \le \dist(F^\star, F_w) - \dist(F_{sw}, F_w)$.
\end{proof}


\paragraph{Discussion of forward KL divergence.}
It is natural to ask, whether can the above proof technique be extended to forward KL?
Our answer is that, we may need an additional assumption.
In our proof, since reverse KL yields \textit{a linear term}, the proof can be carried through. However, forward KL introduces \textit{a logarithmic term}. 
While the Taylor expansions of the log function and a linear term differ only by a remainder term, proving the result requires assuming this remainder is non-negative, and that is why we need an additional assumption like Theorem 3 in~\citep{yao2025understanding}.
Here are the detailed explanations.

Note that
\begin{align} \label{dist_expansion_bregman-for}
    \dist(F_w, g) = \dist(F_{sw}, g) + \dist(F_w, F_{sw}) + \underbrace{\expect_x \left \langle F_w(x)-F_{sw}(x), \nabla \phi(F_{sw}(x)) - \nabla \phi(g(x)) \right \rangle}_{R_2}.
\end{align}
Our goal is to prove that $R_2 \ge 0$.
Now we use forward KL as the loss function in W2SG: $f_{sw} = \argmin_{f \in \cF}\; \dist(F_w, f \circ h_s)$.
In other words, $F_{sw}$ is the \textit{projection} of $F_w$ onto the convex set $V_s$, i.e., $\dist(F_w, g) \ge \dist(F_w, F_{sw})$.
Substitute it into~\cref{dist_expansion_bregman-for} and we have
\begin{align} \label{bregman:disc-for}
    R_2 + \dist(F_{sw}, g) \ge 0.
\end{align}

Again, let $g=F_{sw} + t(F^\star-F_{sw})$, $t \in (0,1)$, $t \to 0^+$.
Consider $\phi(x)=\sum_{i=1}^k x_i \log x_i$, so $\nabla \phi(x)= [\log x_1+1, \cdots, \log x_k+1]^T$.
Using a similar proof technique, we can obtain $R_2=\cO(t)$ and $\dist(F_{sw}, g)=\cO(t^2)$.
Therefore, we know that $R_2 \ge 0$, i.e.,
$$R_2=\expect_x \left \langle F_w(x)-F_{sw}(x), \underbrace{\nabla \phi(F_{sw}(x)) - \nabla \phi(F_{sw} + t(F^\star-F_{sw})(x))}_{\neq \nabla \phi(F_{sw}(x)) - \nabla \phi(F^\star(x))} \right \rangle \ge 0.$$
Consequently, even if we select $g=F^\star$ in~\cref{dist_expansion_bregman-for} and obtain
$$\dist(F_w, g) = \dist(F_{sw}, g) + \dist(F_w, F_{sw}) + \underbrace{\expect_x \left \langle F_w(x)-F_{sw}(x), \nabla \phi(F_{sw}(x)) - \nabla \phi(F^\star(x)) \right \rangle}_{R_3 \neq R_2}.$$
Since we do not know whether $R_3 \ge 0$ is satisfied, we cannot directly prove the desired result.
Since the difference between $R_2$ and $R_3$ can be quantified using exhaustive Taylor expansion,
the nature of proof is similar to the regression analysis of W2SG (Proof of Theorem 3 from~\citet{yao2025understanding}, which introduces an additional assumption for the remainder of Taylor expansion).
However, we do not know whether the remainder is larger than zero. In other words, to prove similar results for forward KL, we may introduce other assumptions like Theorem 3 in~\citet{yao2025understanding}.
In contrast, the success of reverse KL and squared loss is because $R_3 = t \cdot R_2$. In the proof for these reverse losses, if $R_2 \ge 0$, then there also holds $R_3 \ge 0$.
The above discussion indicates that our proof framework cannot be directly extended to the forward KL setting. 
We will explore how to address this limitation in future work.


\paragraph{Extension to reverse cross entropy loss.}
To extend the proof to reverse cross entropy, consider the following theoretical result.

\begin{corollary}
\label{thm:realizable_cross}
Consider W2SG using reverse cross entropy loss:
\begin{align*}
    f_{sw} = \argmin_{f \in \cF_{s}}\; \cross(f \circ h_s, f_w \circ h_w).
\end{align*}
Assume that the function class $\cF_{s}$ is a convex set and $\exists f_s \in \cF_s$ such that $F_s = F^\star$.
Then:
\begin{align*}
    \cross(F^\star, F_{sw}) \le \frac{1}{2}\left(\cross(F^\star, F_w)-\kl(F_{sw}, F_w)\right) + \log k.
\end{align*}
\end{corollary}

If we consider binary classification (such as two famous datasets in AI safety: HH-RLHF~\citep{bai2022training} and CAI-Harmless~\citep{bai2022constitutional}), then $k=2$, making $\log k$ negligible due to the nature of KL divergence $\kl(\cdot, \cdot) \in [0, +\infty)$ and cross-entropy $\cross(\cdot, \cdot) \in [0, +\infty)$.
It shows that if we use reverse cross-entropy loss in W2SG, the strong model's performance is also probably better than weak model's performance, which is also validated in our experiments.

\begin{remark}
The proof also demonstrates that 
$$\cross(F^\star, F_{sw}) \le \cross(F^\star, F_w) - \kl(F_{sw}, F_w) - \epsilon,$$
where $\epsilon=\cross(F^\star,F_{sw}) - \log k$. Due to the same reason, we expect $\epsilon \ge 0$, which comes to the same conclusion.
\end{remark}

\begin{proof}

Rewrite~\cref{dist_expansion_bregman-rev} and we have
\begin{multline} \label{dist_expansion_ce}
    \cross(g, F_w) = \cross(g, F_{sw}) + \cross(F_{sw}, F_w) \\ + \underbrace{\expect_x \left( -H(F_{sw}(x)) + \left \langle g(x)-F_{sw}(x), \nabla \phi(F_{sw}(x)) - \nabla \phi(F_w(x)) \right \rangle \right)}_{R'_1}.
\end{multline}

If we use reverse cross-entropy as the loss function in W2SG: $f_{sw} = \argmin_{f \in \cF}\; \cross(f \circ h_s, F_w)$.
In other words, $\cross(g, F_w) \ge \cross(F_{sw}, F_w)$.
Let $g=F_{sw} + t(F^\star-F_{sw})$, $t \in (0,1)$, $t \to 0^+$. 
Substitute it into~\cref{dist_expansion_bregman-rev} and we have
\begin{align} \label{bregman:disc-ce}
    & R'_1 + \cross(g, F_{sw}) \ge 0, \nonumber \\ \Rightarrow & \underbrace{R_1}_{\cO(t)} + \underbrace{\dist(g, F_{sw})}_{\cO(t^2)} + \expect_x \left( H(g(x))-H(F_{sw}(x)) \right) \ge 0.
\end{align}
Note that 
\begin{align*}
    & \expect_x \left( H(g(x))-H(F_{sw}(x)) \right) 
    \\ = & \expect_x \sum_{i=1}^k [g(x)]_i \log [g(x)]_i - [F_{sw}(x)]_i \log [F_{sw}(x)]_i
    \\ = & \expect_x \sum_{i=1}^k [F_{sw}(x)]_i \log [g(x)]_i + t [F^\star(x)-F_{sw}(x)]_i \log [g(x)]_i - [F_{sw}(x)]_i \log [F_{sw}(x)]_i
    \\ = & \expect_x \sum_{i=1}^k [F_{sw}(x)]_i \log \frac{[g(x)]_i}{[F_{sw}(x)]_i}  + t [F^\star(x)-F_{sw}(x)]_i \log [g(x)]_i
    \\ = & \expect_x \sum_{i=1}^k [F_{sw}(x)]_i \log \left( 1+ t \cdot \frac{[F^\star(x)-F_{sw}(x)]_i}{[F_{sw}(x)]_i} \right)  + t [F^\star(x)-F_{sw}(x)]_i \log [g(x)]_i
    \\ = & \expect_x \sum_{i=1}^k [F_{sw}(x)]_i \left( t \cdot \frac{[F^\star(x)-F_{sw}(x)]_i}{[F_{sw}(x)]_i} + \cO(t^2) \right)  + t [F^\star(x)-F_{sw}(x)]_i \log [g(x)]_i
    \\ = & \expect_x \sum_{i=1}^k t \cdot [F^\star(x)-F_{sw}(x)]_i + \cO(t^2) + t [F^\star(x)-F_{sw}(x)]_i \log [g(x)]_i
    \\ = & \cO(t^2) + t \cdot \expect_x \sum_{i=1}^k [F^\star(x)-F_{sw}(x)]_i \log [g(x)]_i \tag{$\expect_x \sum_{i=1}^k [F^\star(x)-F_{sw}(x)]_i =0$}
    \\ = & \cO(t^2) + t \cdot  [\expect_x H(F_{sw}(x))-\cross(F^\star,F_{sw})] \tag{Definition of entropy and cross entropy},
\end{align*}
where the last inequality is because as $t \to 0^+$, $g \to F_{sw}$.
Consequently, recall~\cref{bregman:disc-ce}, we know that the sum of first-order terms $\cO(t)$ is non-negative, i.e.,
$$t \cdot  [\expect_x H(F_{sw}(x))-\cross(F^\star,F_{sw})] + R_1 \ge 0,$$
which means that
$$\expect_x H(F_{sw}(x)) - \cross(F^\star,F_{sw}) +  \expect_x\left \langle F^\star(x)-F_{sw}(x), \nabla \phi(F_{sw}(x)) - \nabla \phi(F_w(x)) \right \rangle \ge 0.$$

Let $g=F^\star$ in~\cref{dist_expansion_ce} and we obtain
\begin{align*}
    & \cross(F^\star, F_w) = \cross(F^\star, F_{sw}) + \cross(F_{sw}, F_w) - \expect_x H(F_{sw}(x)) \\ & \hspace{5cm} + \expect_x \left \langle F^\star(x)-F_{sw}(x), \nabla \phi(F_{sw}(x)) - \nabla \phi(F_w(x)) \right \rangle
    \\ \Rightarrow & \cross(F^\star, F_w) \ge \cross(F^\star, F_{sw}) + \cross(F_{sw}, F_w) - \expect_x H(F_{sw}(x)) \\  & \hspace{5cm}  + \cross(F^\star,F_{sw}) - \expect_x H(F_{sw}(x))
    \\ \Rightarrow & \cross(F^\star, F_w) \ge \cross(F^\star, F_{sw}) + \kl(F_{sw}, F_w) + \cross(F^\star,F_{sw}) - \expect_x H(F_{sw}(x))
    \\ \Rightarrow & \cross(F^\star, F_w) \ge \cross(F^\star, F_{sw}) + \kl(F_{sw}, F_w) + \cross(F^\star,F_{sw}) - \log k \tag{$H(F_{sw}(x)) \le \log k$}
\end{align*}

Therefore, we prove the result
$$\cross(F^\star, F_{sw}) \le \\ \frac{1}{2}\left(\cross(F^\star, F_w)-\kl(F_{sw}, F_w)\right) + \log k.$$
\end{proof}

\subsection{Proof of~\cref{thm:non-realizable-finite-samples}} \label{proof_non-realizable}

\paragraph{Proof sketch.}
By defining nine variables associated with given models, we substitute key components in the proof of~\cref{thm:realizable} to derive a set of inequalities among these variables.
Through a series of carefully designed transformations, we reformulate the triangle-like inequalities involving three remainder terms.
Ultimately, leveraging tools from statistical learning theory, several inequalities in information-theoretic analysis, and the properties of specific functions, we sequentially demonstrate that these three remainder terms become infinitesimal as $n \to \infty$ and $\epsilon \to 0$.

Let $\dist(\cdot, \cdot)$ be $\kl(\cdot, \cdot)$.
For a clear presentation, denote
\begin{align*}
    A &= \dist(F_s, F_{sw})\\
    B &=\dist(F_{sw}, F_w) \\
    C &= \dist(F_s, F_w)\\
    D &= \dist(F^\star, F_s) =\eps\\
    E &= \dist(F^\star, F_{sw}) \\
    F &= \dist(F^\star, F_w) \\
    G &= \dist(F^\star, \hat{F}_{sw}) \\
    H &= \dist(\hat{F}_{sw}, F_{sw}) \\
    I &= \dist(\hat{F}_{sw}, F_w).
\end{align*}

Now we start the proof of~\cref{thm:non-realizable-finite-samples}.
A uniform convergence result and two claims used in the proof are provided at the end of the proof.
The proof is strongly motivated by Theorem 4 in~\citet{yao2025understanding}.
While our work primarily focuses on classification, their Theorem 4 is specifically centered on regression.

\begin{proof}

Note that by virtue of the range of $f^\star, f_w$ and all functions in $\cF$ being absolutely bounded, and $\dist$ is also bounded.

Due to $F^\star \notin V_s$, we replace $F^\star$ with $F_s$ in the final step of proof of~\Cref{thm:realizable}, we obtain
\begin{align}
    C \ge A + B. \label{eqn:1}
\end{align}

Recall that $\left \langle f,g \right \rangle_E \triangleq \bE_{x \sim \cP} [f(x)^Tg(x)]$, which is used here for a clear presentation. So we have
\begin{align*}
    E & = A + D - \expect_x \sum_{i=1}^k ([F^\star(x)]_i-[F_s(x)]_i)\log \frac{[F_{sw}(x)]_i}{[F_s(x)]_i}
    \\ & = A + D - \underbrace{\left \langle F^\star-F_s, \log \frac{F_{sw}}{F_s} \right \rangle_E}_{t_1}.
\end{align*}
The $\log$ here is element-wise.
Using the similar notation, we have the following 
\begin{align}
    & E = A + D - \underbrace{\left \langle F^\star-F_s, \log \frac{F_{sw}}{F_s} \right \rangle_E}_{t_1}, \label{eqn:2}\\ 
    & F = C + D - \underbrace{\left \langle F^\star-F_s, \log \frac{F_w}{F_s} \right \rangle_E}_{t_2}, \label{eqn:2-1} \\
    & G = E - H - \underbrace{\left \langle \hat{F}_{sw}-F^\star, \log \frac{F_{sw}}{\hat{F}_{sw}} \right \rangle_E}_{t_3} \label{eqn:2-2}. 
\end{align}

Combining \eqref{eqn:1} and \eqref{eqn:2}, we get
\begin{align}
    E \le C + D - B - t_1. \label{eqn:3}
\end{align}

By a uniform convergence argument (\Cref{lem:uniform-convergence}), we have that with probability at least $1-\delta$ over the draw of $\{(x_1,y_1),\dots, (x_n,y_n)\}$ that were used to construct $\hat{F}_{sw}$,
\begin{align}
    I &\le B + \underbrace{\cO\left(\sqrt{\frac{\cC_{\cF_s}}{n}}\right)}_{t_4} + \underbrace{\cO\left(\sqrt{\frac{\log(1/\delta)}{n}}\right)}_{t_5}. \label{eqn:uc}
\end{align}

Combining \eqref{eqn:3} with \eqref{eqn:uc} and we have
\begin{align}
    E \le C + D - I - t_1 + t_4 + t_5. \label{eqn:4}
\end{align}

Combining \eqref{eqn:2-1} with \eqref{eqn:4} and we have
\begin{align}
    E \le F - I - t_1 + t_2 + t_4 + t_5. \label{eqn:5-1}
\end{align}

Combining \eqref{eqn:2-2} with \eqref{eqn:5-1} and we have
\begin{align}
    G \le F - I - H - t_1 + t_2 - t_3 + t_4 + t_5. \label{eqn:5}
\end{align}

We replace $F^\star$ with $\hat{F}_{sw}$ in the final step of proof of~\Cref{thm:realizable} and obtain:
\begin{align}
    I \ge H + B. \label{eqn:uc-projection}
\end{align}

Combining \eqref{eqn:uc-projection} with \eqref{eqn:uc} and we have
\begin{align}
    0 \le H \le t_4 + t_5 = \cO\left(\sqrt{\frac{\cC_{\cF_s}}{n}}\right) + \cO\left(\sqrt{\frac{\log(1/\delta)}{n}}\right). \label{eqn:6}
\end{align}

Combining \eqref{eqn:6} with \eqref{eqn:5} and we have
\begin{align}
    G \le F - I - t_1 + t_2 - t_3 + t_4 + t_5. \label{eqn:7}
\end{align}

While $t_4$ and $t_5$ are known in~\eqref{eqn:uc}, we analyze $t_1$, $t_2$ and $t_3$ one by one.

\paragraph{Deal with $t_1$.}

We know that
\begin{align}
    t_1 & = \left \langle F^\star-F_s, \log \frac{F_{sw}}{F_s} \right \rangle_E. \nonumber
\end{align}
Using the fact that $\frac{F_{sw}(x)}{F_s(x)} \le \frac{1}{\gamma}$, we have
\begin{align} \label{def_t_1_ineq}
    |t_1| & \le \frac{1}{\gamma} \expect_x \sum_{i=1}^k \left| [F^\star(x)]_i - [F_s(x)]_i \right| \nonumber
    \\ & = \frac{2}{\gamma} \expect_x \tv(F^\star(x), F_s(x)) \tag{Definition of TV distance} \nonumber
    \\ & \le \frac{2}{\gamma} \expect_x \sqrt{\frac{1}{2}\mathrm{D}_{\mathrm{KL}}(F^\star(x) \| F_s(x))} \tag{Pinsker's inequality} \nonumber
    \\ & \le \frac{2}{\gamma} \sqrt{\frac{1}{2} \expect_x \mathrm{D}_{\mathrm{KL}}(F^\star(x) \| F_s(x))} \tag{Jensen’s inequality} \nonumber
    \\ & = \frac{2}{\gamma} \sqrt{\frac{1}{2} \dist(F^\star, F_s)} \tag{Definition of $\dist$} \nonumber
    \\ & = \frac{1}{\gamma} \sqrt{2 \varepsilon}
\end{align}

Therefore,
\begin{align} \label{equation_t_1}
    |t_1| = \cO(\sqrt{\varepsilon}).
\end{align}

\paragraph{Deal with $t_2$.}
The proof for $t_2$ is similar for $t_1$.
In particular, replacing $F_{sw}$ with $F_w$ in the above and we can get 
\begin{align} \label{equation_t_2}
    |t_2| = O(\sqrt{\varepsilon}).
\end{align}

\paragraph{Deal with $t_3$.}

We know that
$$t_3 = \left \langle \hat{F}_{sw}-F^\star, \log \frac{F_{sw}}{\hat{F}_{sw}} \right \rangle_E = \expect_x \sum_{i=1}^k ([\hat{F}_{sw}(x)]_i-[F^\star(x)]_i)\log \frac{[F_{sw}(x)]_i}{[\hat{F}_{sw}(x)]_i}.$$

According to~\cref{lem:uniform-convergence}, 
with probability at least $1-\delta$ over the draw of $(x_1,y_1),\dots,(x_n, y_n)$, we have
\begin{align} \label{t_3_proof_unif_conv}
    \left|\dist(\hat{F}_{sw}, F_w) - \dist(F_{sw}, F_w) \right| \le \cO\left(\sqrt{\frac{\cC_{\cF}}{n}}\right) + \cO\left(\sqrt{\frac{\log(1/\delta)}{n}}\right).
\end{align}
Notice that 
\begin{align} \label{t_3_proof}
H & = \dist(F_{sw}, \hat{F}_{sw}) \nonumber \\ & = \dist(F_w, F_{sw}) - \dist(F_w, \hat{F}_{sw}) + \left \langle F_w+F_{sw}, \log \frac{F_{sw}}{\hat{F}_{sw}} \right \rangle_E.
\end{align}
Substitute~\eqref{eqn:6} and~\eqref{t_3_proof_unif_conv} into~\cref{t_3_proof} with the triangle inequality for absolute values, we get
\begin{align*}
    \left| \left \langle F_w+F_{sw}, \log \frac{F_{sw}}{\hat{F}_{sw}} \right \rangle_E \right| \le \cO\left(\sqrt{\frac{\cC_{\cF}}{n}}\right) + \cO\left(\sqrt{\frac{\log(1/\delta)}{n}}\right)
\end{align*}
Since $|F_w(x)+F_{sw}(x)|$ is lower bounded, we have
$$\left| \left \langle \vec{1}, \log \frac{F_{sw}}{\hat{F}_{sw}} \right \rangle_E \right| \le \cO\left(\sqrt{\frac{\cC_{\cF}}{n}}\right) + \cO\left(\sqrt{\frac{\log(1/\delta)}{n}}\right).$$

Since $|\hat{F}_{sw}(x)-F^\star(x)|$ is upper bounded, there holds
\begin{align} \label{equation_t_3}
    |t_3| = \left| \left \langle \hat{F}_{sw}-F^\star, \log \frac{F_{sw}}{\hat{F}_{sw}} \right \rangle_E \right| \le \cO\left(\sqrt{\frac{\cC_{\cF}}{n}}\right) + \cO\left(\sqrt{\frac{\log(1/\delta)}{n}}\right).
\end{align}

Therefore, combing~\eqref{equation_t_1}, \eqref{equation_t_2} and \eqref{equation_t_3}, we have
\begin{align} \label{t_1_2_3_ineq}
    |t_1| + |t_2| + |t_3| \le O(\sqrt{\varepsilon}) + \cO\left(\sqrt{\frac{\cC_{\cF}}{n}}\right) + \cO\left(\sqrt{\frac{\log(1/\delta)}{n}}\right).
\end{align}

Finally, combing~\eqref{eqn:uc} and~\eqref{eqn:7} with \eqref{eqn:6} and~\eqref{t_1_2_3_ineq}, we get the result:
\begin{align*}
\dist(F^\star, \hat{F}_{sw}) \le \dist(F^\star, F_w) - \dist(\hat{F}_{sw}, F_w) + O(\sqrt{\eps}) + \cO\left(\sqrt{\frac{\cC_{\cF}}{n}}\right) + \cO\left(\sqrt{\frac{\log(1/\delta)}{n}}\right),
\end{align*}
where in the last inequality, we instantiate asymptotics with respect to $\eps \to 0$ and $n \to \infty$.

\end{proof}

Here are some tools used in the above proof.

\begin{lemma}[Uniform convergence]
\label{lem:uniform-convergence}
Let $(x_1,y_1),\dots,(x_n, y_n)$ be an i.i.d. training sample, where each $x_i \sim \cP$ and $y_i = F_w(x_i)$ for a target function $F_w$. For a fixed strong model representation $h_s$, we employ reverse KL loss in W2SG:
\begin{align*}
    & f_{sw} = \argmin_{f \in \cF_{s}}\; \dist(f \circ h_s, F_w) = \argmin_{f \in \cF_{s}} \; \bE_{x \sim \cP} \left[ \sum_{i=1}^k [f \circ h_s(x)]_i \log \frac{[f \circ h_s(x)]_i}{[F_w(x)]_i} \right],
    \\ & \hat{f}_{sw} = \argmin_{f \in \cF_s} \disthat(f \circ h_s, F_w) = \argmin_{f \in \cF_{s}} \; \frac{1}{n} \sum_{j=1}^n \left[ \sum_{i=1}^k [f \circ h_s(x_j)]_i \log \frac{[f \circ h_s(x_j)]_i}{[F_w(x_j)]_i} \right].
\end{align*}
Assume that the range of $F_w$ and functions in $\cF_s$ is absolutely bounded. Then, with probability at least $1-\delta$ over the draw of $(x_1,y_1),\dots,(x_n, y_n)$, we have
\begin{align*}
    \left|\dist(\hat{F}_{sw}, F_w) - \dist(F_{sw}, F_w) \right| \le \cO\left(\sqrt{\frac{\cC_{\cF_s}}{n}}\right) + \cO\left(\sqrt{\frac{\log(1/\delta)}{n}}\right),
\end{align*}
where $\cC_{\cF_s}$ is a constant capturing the complexity of the function class $\cF_s$.
\end{lemma}

\begin{proof}
The proof follows lemma 4 in~\citet{yao2025understanding}.
Swap the order of the two elements in $\dist(\cdot, \cdot)$ and $\hat{L}_\cP(\cdot, \cdot)$ in their proof and we can prove the result.
\end{proof}

\begin{claim}[\citet{yao2025understanding}] \label{claim_xlnx}
Let $f(x), g(x) \in [\gamma,1]$ where $\gamma>0$. If there exists $\xi>0$ such that $\int_{\cX} \left|f(x)-g(x) \right| d x \le \xi$,
then there holds
$$\int_{\cX} \left| \log f(x)- \log g(x) \right| d x \le \frac{1}{\gamma}\xi.$$
\end{claim}

\begin{claim}[\citet{yao2025understanding}] \label{claim_xlnx_reverse}
Let $f(x), g(x) \in [\gamma,1]$ where $\gamma>0$. 
If there exists $\xi>0$ such that $\int_{\cX} \left| \log f(x) - \log g(x) \right| d x \le \xi$,
then there holds
$$\int_{\cX} \left| f(x)- g(x) \right| d x \le \xi.$$
\end{claim}

\section{Additional Experimental Details and Results}

We first provide a detailed explanation of the evaluation metric.
To determine the effectiveness of a model $F$ in distinguishing between the selected and rejected completions ($y_c$ and $y_r$) for a given prompt $x$, we require that $F$ ranks the chosen completion higher than the rejected one.
This condition is formulated as $F(y_c)-F(y_r)>0$ for each pair $\Tilde{x}=(x;y_c,y_r)$, implying that $F(\Tilde{x})$ should exceed $0.5$.
Consequently, the test accuracy is defined as the fraction of instances where $F(\Tilde{x})>0.5$.

\subsection{Results of Pythia} \label{exp_result_pythia}

The overall trends observed in~\cref{fig:pythia} are similar with those in~\cref{fig:cai}. Our analysis of the results in~\cref{fig:pythia} further highlights those insights:
First, the accuracy exhibits a consistent upward trend from left to right, reinforcing the finding that the generalization capability of the strong model improves when a more capable weak model is utilized as the supervisor.
Second, the results demonstrate that in the majority of experimental settings (7 out of 12), reverse losses outperform forward losses, leading to stronger model performance.
Given the superior capabilities of the Pythia series compared to the GPT-2 series~\citep{biderman2023pythia}, as well as the fact that Pythia's strong ceiling model outperforms GPT-2, a key implication emerges. When the Pythia series serves as a weak model, it may generate less noise on non-target labels. As a result, the potential advantages of reverse losses are diminished, leading to only a slight improvement of reverse losses over forward losses.
Finally, across almost all of the settings (10 out of 12), the strong model trained with reverse KL and CE losses achieves superior performance compared to its weak supervisor. This observation is in full agreement with our theoretical predictions, further validating the effectiveness of reverse losses in enhancing model performance.

\begin{figure*}[t]
\begin{center}
\vspace{-10pt}
\subfigure[Results of Pythia-series on CAI-Harmless]{ 
\begin{minipage}[t]{0.95\linewidth}  \centerline{\includegraphics[width=1\linewidth]{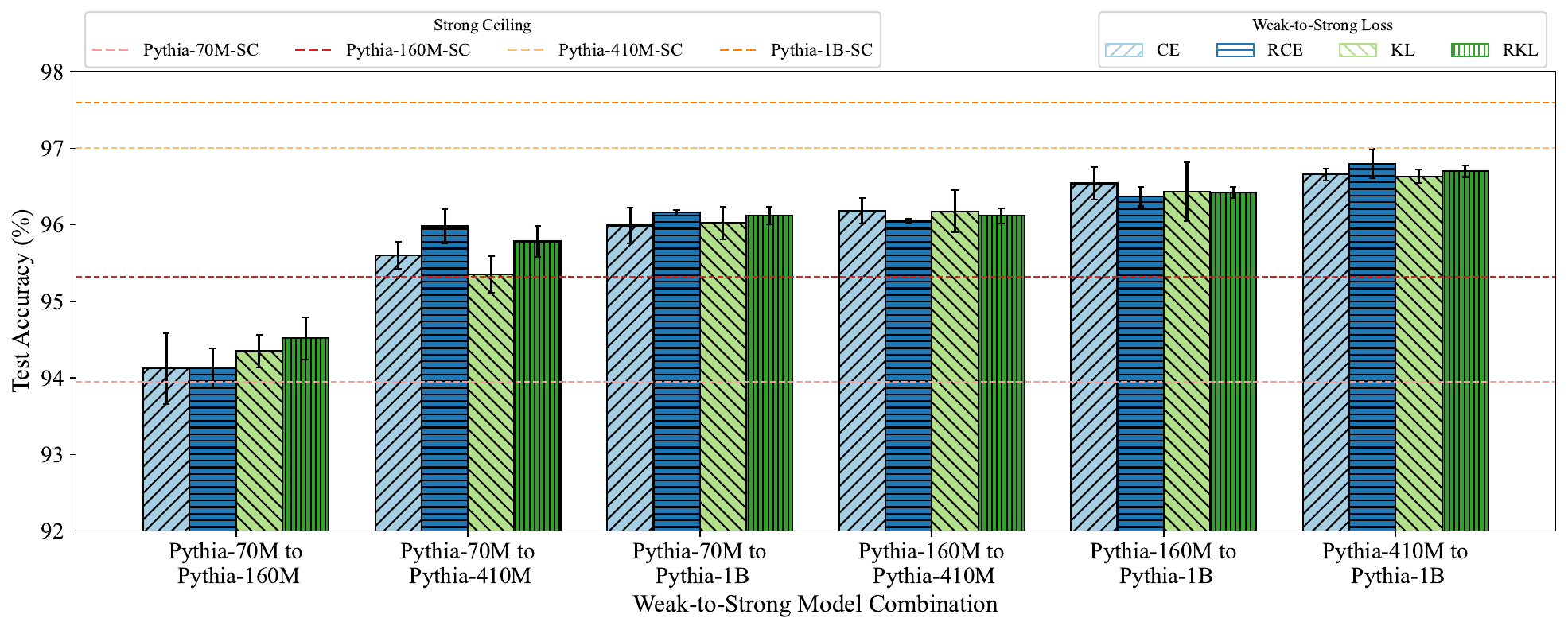}}
\end{minipage}  
}  
\subfigure[Results of Pythia-series on helpful set of HH-RLHF]{
\begin{minipage}[t]{0.95\linewidth}
\centerline{\includegraphics[width=1\linewidth]{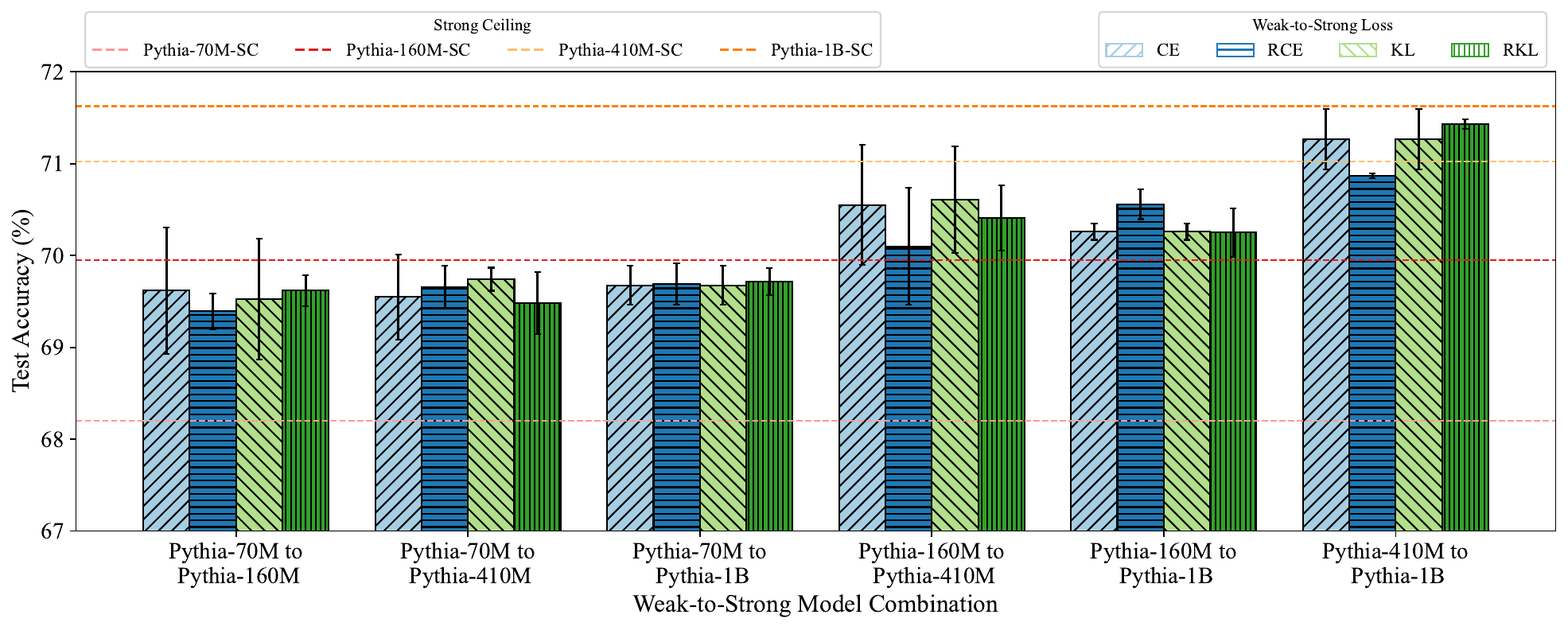}}
\end{minipage}  
}    
\vspace{-5pt}
\caption{Results of Pythia-series. ``SC'' denotes the strong ceiling model, and ``A to B'' indicates the use of weak teacher ``A'' to supervise strong student ``B''. The terms CE, RCE, KL, and RKL refer to cross-entropy loss, reverse cross-entropy loss, forward KL divergence loss, and reverse KL divergence loss, respectively. Error bars represent the standard deviation across three runs of the experiment.}
\label{fig:pythia}
\end{center}
\vspace{-10pt}
\end{figure*}

\subsection{Auxiliary Confidence Loss} \label{exp:conf_loss}

As highlighted by~\citet{burns2023weak}, we explore an alternative approach: introducing an additional regularization term designed to enhance the strong model’s confidence in its predictions using standard cross-entropy loss, which is called ``Auxiliary Confidence Loss'' in~\citet{burns2023weak}:
\begin{align} \label{eq:confidence_loss}
    L_{\mathrm{conf}}(f)=(1-\alpha) \cdot \underbrace{\cross \left(F_w, f \circ h_s\right)}_{\text{vanilla cross-entropy loss}} + \; \alpha \cdot \underbrace{\cross\left(\hat{f}_t \circ h_s, f \circ h_s \right)}_{R(f)},
\end{align}
where $\alpha$ is the weight constant, $R(f)$ is the regularization term, and $\hat{f}_t$ corresponds to hardened strong model predictions using a threshold $t$, i.e., for any $x$:
\begin{align*}
    \hat{f}_t \circ h_s(x)= \mathbb{I}(f \circ h_s(x) > t) \in \{ 0,1 \},
\end{align*}
where $\mathbb{I}(\cdot)$ is the indicator function.
Rewrite~\cref{eq:confidence_loss} as the minimization objective in W2SG:
\begin{align}
    f_{sw} = \argmin_{f \in \cF_{s}}\; L_{\mathrm{conf}}(f).
\end{align}
As advocated by~\citet{burns2023weak}, this regularization serves to mitigate overfitting to weak supervision, thereby improving the overall performance of the strong model.
Therefore, to further explore the advantage of reverse cross-entropy loss, we replace the vanilla cross-entropy with reverse cross-entropy in $L_{\mathrm{conf}}(f)$ and conduct W2SG using the following objective:
\begin{align} \label{eq:reverse_confidence_loss}
    f_{sw}^r & = \argmin_{f \in \cF_{s}}\; L^r_{\mathrm{conf}}(f) \nonumber
    \\ & = \argmin_{f \in \cF_{s}}\; (1-\alpha) \cdot \underbrace{\cross \left(f \circ h_s, F_w\right)}_{\text{reverse cross-entropy loss}}+ \; \alpha \cdot R(f).
\end{align}

We set $\alpha=0.2$ to ensure that the reverse/forward CE loss dominates the regularization, because we use a small batch size here and we want to reduce the negative impact of the randomness and instability brought by the auxiliary confidence loss within a single batch.
The experimental comparison between $f_{sw}$ and $f_{sw}^r$ is presented in~\cref{fig:conf_loss}.
First, by combining the observations from~\cref{fig:cai} and~\cref{fig:conf_loss}, we observe that the application of auxiliary confidence loss slightly enhances the performance of the strong model, consistent with the findings of~\citet{burns2023weak}.
Second, the use of reverse cross-entropy loss consistently enables the strong model to outperform its counterpart trained with standard cross-entropy loss. This finding, combined with previous experimental results in this work, highlights the superior effectiveness of reverse losses compared to forward losses.

\end{document}